\documentclass[english,11pt]{article}

\usepackage[utf8]{inputenc}
\usepackage{microtype}
\usepackage{enumerate}
\usepackage{xcolor}
\usepackage{amsmath}
\usepackage{amsthm}
\usepackage{amssymb}
\usepackage{graphicx}
\usepackage{epstopdf}
\usepackage{algorithm}
\usepackage{algorithmic,color}
\usepackage{multirow}
\usepackage{appendix}
\usepackage{color}
\usepackage[a4paper,left=2.5cm,right=2.5cm,top=3.cm,bottom=3.cm]{geometry}
\usepackage{hyperref}

\newcommand{\N}{\mathbb{N}}

\newcommand{\R}{\mathbb{R}}

\newcommand{\E}{\mathbb{E}}
\newcommand{\Prob}{\mathbb{P}}

\newtheorem{thm}{Theorem}[section]
\newtheorem{thmnosec}{Theorem}[]
\newtheorem{proposition}[thm]{Proposition}

\newtheorem{cor}[thm]{Corollary}
\theoremstyle{remark}

\theoremstyle{definition}

\newtheorem{assumption}[thmnosec]{Assumption}

\DeclareMathOperator*{\argmin}{arg\,min}
\newcommand{\tr}{\operatorname{tr}} 

\def\<{\langle}
\def\>{\rangle}

\DeclareMathSymbol{\shortminus}{\mathbin}{AMSa}{"39}

\begin{document}


  \author{Luca Ratti
  \thanks{Università degli Studi di Bologna, Piazza di Porta S. Donato, 5, 40126 Bologna, email: luca.ratti5@unibo.it}}
  \title{\bf Learned reconstruction methods for inverse problems: \\ sample error estimates}
  \date{ }
  
\maketitle

\abstract{
Learning-based and data-driven techniques have recently 
become a subject of primary interest in the field of reconstruction and regularization of inverse problems. 
Besides the development of novel methods, yielding excellent results in several applications, their theoretical investigation has attracted growing interest, e.g., on the topics of reliability, stability, and interpretability.\\
In this work, a general framework is described, allowing us to interpret many of these techniques in the context of statistical learning. This is not intended to provide a complete survey of existing methods, but rather to put them in a working perspective, which naturally allows their theoretical treatment. The main goal of this dissertation is thereby to address the generalization properties of learned reconstruction methods, and specifically to perform their sample error analysis. This task, well-developed in statistical learning, consists in estimating the dependence of the learned operators with respect to the data employed for their training. A rather general strategy is proposed, whose assumptions are met for a large class of inverse problems and learned methods, as depicted via a selection of examples.
}
\section{Introduction}
\label{sec:intro}

The mathematical treatment of inverse problems has proved to be a lively and attractive research field, driven and motivated by a wide variety of applications and by the theoretical challenges induced by their ill-posed nature. In order to provide more accurate and reliable strategies, especially for the reconstruction task, the scientific research in the field has shown a growing interest in the topic of \textit{learned reconstruction}, or data-driven, methods, to combine classical, model-based approaches with valuable information of statistical nature. This has represented a natural outcome and development of the analysis of inverse problems, both on a numerical and on a theoretical side: indeed, the idea of leveraging prior knowledge on the solution has traditionally been considered to mitigate ill-posedness, as a regularization tool as much as a support for the reconstruction. We have now witnessed the emergence of several learning-based approaches to inverse problems, providing, in many cases, striking numerical results in terms of accuracy and efficiency. Moreover, significant interest has grown in the direction of theoretical guarantees for such techniques, spanning from the demand of interpretability and reliability, up to the issues of stability and convergence \cite{arridge2019, mukherjee2020learned}. Despite several distinct avenues have emerged, which are now well-established and are developing independently (to name a few: generative models, unrolled techniques, Plug and Play schemes), it is possible to provide a unifying overview of them from the point of view of statistical learning theory \cite{cucker2002mathematical}.
\par
In this context, the goal pursued by this paper is twofold. On the one side, it aims to provide a general theoretical framework in statistical learning that is able to comprehend a large family of data-driven reconstruction methods. This can be accomplished by first establishing the common rationale behind data-driven regularization, also exploring analogies and differences with existing techniques, and then populating such a framework with a (non-exhaustive, but convincing) review of approaches in literature.
The second task is instead to contribute to the study of generalization estimates for data-driven reconstruction methods, specifically in terms of sample error estimates: this is a classical goal of statistical learning techniques, and it accounts for quantifying how much the learned operator is influenced by the dataset which has been used to train it. To do so, we rely on solid and rather modern techniques such as covering and chaining, which allow us to treat learning problems in (possibly infinite-dimensional) functional spaces. 
\par
The desired outcome of the discussion proposed in this work is to stimulate growing attention on the topics of generalization and sample error estimates for learned reconstruction methods, providing a comprehensive setup and a (hopefully) useful and spendable result that can be applied to a large family of existing approaches and might inspire further extensions. 
\par
The structure of the paper is as follows:
\begin{itemize}
    \item Sections \ref{sec:statistics_inv_prob} and \ref{sec:framework} outline a rather general formulation of learned reconstruction methods, providing the coordinates to describe such a paradigm within the inverse problems and statistical learning realms, respectively.
    \item In Section \ref{sec:learnedreg_review},
    a limited review of existing techniques is exposed. Without proposing an exhaustive taxonomy of all approaches, some general guidelines are provided to envisage the main trends in literature within the scope of the proposed framework.
    \item Section \ref{sec:errors} sets the stage for the analysis of the generalization error, recalling the relevant tools and discussing the limitations of existing results.
    \item In Section \ref{sec:main}, the main theoretical results of this work are reported. They hold under rather general conditions on the employed statistical model, and they can be applied in an infinite-dimensional setup. This is often beneficial when dealing with inverse problems, which are naturally formulated in a continuous setup, as it guarantees discretization-invariant theoretical guarantees.
    \item Finally, Section \ref{sec:examples} endows the proposed analysis with some more specific examples satisfying the assumptions yielding the derived error estimates. A general strategy to verify them is proposed, together with some specific examples drawn from techniques that are present in the literature.
\end{itemize}

\section{Learned reconstruction methods: an inverse problems perspective}
\label{sec:statistics_inv_prob}

The goal of this section is to introduce a rigorous formulation of learned reconstruction methods for inverse problems, describing a unified framework including many techniques already present in literature (see Section \ref{sec:learnedreg_review}).
The discussion consists of three steps: in Section \ref{ssec:deterministic}, the \textit{deterministic} setting of inverse problems is recapped, highlighting its connection with (classical) regularization theory. Section \ref{ssec:statisticIP} is instead devoted to the description of a statistical approach to inverse problems, rooted in both statistical learning (see \cite{cucker2002mathematical}) and estimation theory, which provides a sound framework for many modern techniques known as \textit{learned reconstruction}, or \textit{data-driven} methods. This setup, which is the main environment for the theoretical discussion of this work, has several connections with other approaches, which have been developed to connect statistics and inverse problems. The analogies and differences between those approaches are the focus of Section \ref{ssec:otherstat}.

\subsection{A \textit{deterministic} approach to inverse problems: classical regularization techniques}
\label{ssec:deterministic}
Let $X,Y$ be Banach spaces, and consider $x \in X$ and $\varepsilon \in Y$ such that $\|\varepsilon\|_Y \leq \delta$.
Let $F \colon X \rightarrow Y$ a (possibly nonlinear) operator and
consider the inverse problem:
\begin{equation}
 \text{Given} \quad y = F(x) + \varepsilon, \qquad \text{recover}\quad x. 
    \label{eq:IP}
\end{equation}
Solving \eqref{eq:IP} is in general an \textit{ill-posed} problem. Let us focus, for example, on linear inverse problems, i.e., when $F = A\colon X \rightarrow Y$ is a bounded linear operator. In this case, one way to retrieve $x$ from $y$ is to apply a left-inverse operator $A^{-1}$. Unfortunately, such an operator, in general, may have a domain $\mathcal{D}(A^{-1})$ strictly smaller than $Y$, may be non-unique, and may be unbounded. The first aspect would not be an issue in the presence of noiseless data, since $Ax \in \operatorname{Im}(A) \subset \mathcal{D}(A^{-1})$; nevertheless, the perturbation $\varepsilon$ added to the noise may entail that $y \notin \mathcal{D}(A^{-1})$, ultimately resulting in a lack of \textit{existence} of a solution to \eqref{eq:IP}. Secondly, whenever the null-space $\operatorname{ker}(A)$ is non-trivial, the existence of multiple left-inverse operators implies also the \textit{non-uniqueness} of the solution. Finally, even in cases in which existence and uniqueness are not problematic, the inverse $A^{-1}$ can be an unbounded operator (suppose, for example, that $A$ is a compact, dense-range operator). In this case, small perturbations on the datum $y$ (such as the one induced by $\varepsilon$) may lead to tremendously different reconstructions $x$ - an issue known as \textit{instability}. These drawbacks are in general even worse when considering nonlinear inverse problems, in which an expression for the inverse of $F$ might be in general unknown.
\par
To overcome the ill-posedness of inverse problems, a common remedy is represented by regularization. As formalized in \cite{engl1996}, a \textit{regularizer} for problem \eqref{eq:IP} is a family of (possibly nonlinear) operators $\{R_\alpha\}_\alpha$ such that for every $\alpha > 0$ the operator $R_\alpha$ is continuous and there exists a parameter choice rule $\alpha(\delta)$ ensuring that, for any $x \in X$, as the noise level $\delta \rightarrow 0$, the regularized solution $R_{\alpha(\delta)}(F(x)+\varepsilon)$ converges to $x$, uniformly in $\varepsilon$ satisfying $\| \varepsilon\| \leq \delta$. The effectiveness of a regularization strategy strongly depends on the definition of the operators $R_\alpha$, which should take into account both theoretical requirements and practical insights. In particular, the most successful strategies are the ones that take advantage of any possible prior knowledge regarding the solution of \eqref{eq:IP}. Indeed, even though $x$ is unknown, some additional information about it may be accessible: e.g., $\|x\|_X$ is expected to be small, $x$ is likely to be similar to a reference object $x_0$, or only a few components of $x$ are non-null, in a suitable representation. Encoding such features into the design of $R_\alpha$ is a crucial task of the analytical theory of regularization. 
\par
Among the many different possibilities, let us mention \textit{variational} regularization, which will also provide important examples in the subsequent sections. A variational regularizer is in general defined as
\[
R_\alpha(y) = \argmin_{x \in X} \{ J_\alpha(x;y) \} = \argmin_{x \in X} \big\{ d_Y(y,F(x)) + \alpha \Psi(x)\big\},
\]
namely, as the solution of a minimization problem, associated with a functional $J_\alpha$ consisting of two terms: $d_Y(y,F(\cdot))$, being $d_Y$ a metric in $Y$, which measures the consistency of a candidate solution $x$ with the measurement $y$, and a regularization term $\Psi \colon X \rightarrow \R$. The goal of $\Psi$ is to penalize undesired features in $x$, so that the minimizer of $J_\alpha$ will balance a good data fidelity (namely, a low score of $d_Y(y,F(x))$) and a good accordance with prior knowledge (i.e., a low value of $\Psi(x)$). The balance of the two terms is controlled by the parameter $\alpha$, which is supposed to vanish as $\delta \rightarrow 0$: data fidelity is prioritized as the measurements $y$ get more and more accurate. In variational regularization, the task of designing $R_\alpha$ is then moved to the selection of a functional $\Psi$. In the most classical example of Tikhonov regularization, the choice $\Psi(x) = \frac{1}{2}\|x\|_X^2$ reflects the idea that the desired solution is supposed to have a small norm. More advanced choices of $\Psi$ allow the incorporation of more complicated and refined properties of the solution, such as smoothness, sparsity, and possible constraints: for a broad picture of regularization techniques, we refer, e.g., to \cite{kaltenbacher2008iterative,schuster2012regularization,benning2018modern}.

\subsection{A \textit{statistical learning} approach to inverse problems: learned reconstruction techniques}
\label{ssec:statisticIP}

It is sometimes meaningful, and useful, to formulate the inverse problem \eqref{eq:IP} in a statistical setup. In particular, let us introduce a probability space $(\Omega, \mathcal{B}, \mathbb{P})$, and consider two random variables $x$ and $\varepsilon$ taking values on the Banach spaces $X$ and $Y$, respectively, according to the following requirements.
\begin{assumption} The random variables $x,\varepsilon$ are independent and satisfy:
    \begin{itemize}
        \item the covariance operators $\Sigma_x,\Sigma_\varepsilon$ are trace-class;
        \item the noise $\varepsilon$ is zero-mean: $\E[\varepsilon]=0$, and $\tr(\Sigma_\varepsilon) = \operatorname{Var}[\varepsilon] = \E[\| \varepsilon\|^2_Y] \leq \delta^2$.
    \end{itemize}
    \label{ass:stat}
\end{assumption}
The choice to model the noise as a random variable $\varepsilon\colon \Omega \rightarrow Y$ reflects the assumption, quite common in applied sciences, that the measurement process encoded by the operator $F$ is affected by a stochastic perturbation. On the contrary, describing the unknown $x$ as a random variable might be less intuitive. This is quite common in the statistical learning framework (for a comprehensive overview of the approach, see e.g. \cite{cucker2002mathematical}), in which one considers not only the problem of reconstructing a single ground truth associated with a specific measurement, but of retrieving the whole map from $Y$ to $X$ that connects the variables $y$ and $x$.\footnote{When confronting the statistical learning literature, observe that the role of $x$ and $y$ is usually swapped, but this is a consequence of dealing with inverse problems!}
\par
This represents a paradigm shift with respect to the usual setup of inverse problems: on the one hand, the goal of the problem is not to recover a single $x$ as in \eqref{eq:IP}, but to retrieve (an approximation of) the whole inverse operator $F^{-1}$; on the other hand, the random variable $x$ is no longer an unknown of the problem, but its statistical model is supposed to be, at least partially, known. 
Full information regarding the ground truth $x$ would be encoded by its probability distribution $\rho_x$ on $X$, but partial information can be extracted, e.g., from an i.i.d. sample $\{x_j\}_{j=1}^m$ drawn from $\rho_x$. This somehow extends the concept of \textit{prior knowledge}, which is already crucial in classical regularization of inverse problems: $\rho_x$ describes which are the most likely values of $x$, which kind of features are most common in every realization, and so on. The motivation for this approach comes, obviously, from the idea that large datasets of ground truths may be available for many applications of interest. For example, when considering an imaging problem, e.g., Computed Tomography applied to neurology, one might rely on large datasets of brain images to extract insights on how, most likely, a ground truth image should look like.
\par
To formulate the statistical learning approach to inverse problems,
let us first introduce the joint probability distribution $\rho$ on $X \times Y$ such that $(x,y) \sim \rho$. Notice that $\rho$ is uniquely determined by the probability distributions $\rho_x,\rho_\varepsilon$ of $x$ and $\varepsilon$ (which are independent by Assumption \ref{ass:stat}), and by the forward model $y = F(x) + \varepsilon$: letting $\tilde{F} \colon X\times Y \rightarrow X \times Y$ s.t. $\tilde{F}(x,\varepsilon) = [x; F(x)+\varepsilon]$,
\begin{equation}
    \rho = \tilde{F}_*(\rho_x \times \rho_\varepsilon),
    \label{eq:rho}
\end{equation}
namely, it is the push-forward of the product measure of $x$ and $\varepsilon$ through the operator $\tilde{F}$. Finally, it is possible to formulate the statistical learning approach to inverse problems as follows:
\begin{equation}
\begin{gathered}
    \text{ Given (partial knowledge of) the distribution $\rho$}, \\ \text{recover $R \colon Y \rightarrow X \quad$ s.t.} \quad R(y) = R(F(x)+\varepsilon) \approx x
\end{gathered}
\label{eq:statIP}
\end{equation}
For problem \eqref{eq:statIP} to be rigorously formulated, it is necessary to clarify at least two aspects, namely, the meaning of \textit{partial} knowledge of the joint distribution $\rho$, and in which sense $R(y)$ should approximate $x$. These are the topics of Sections \ref{ssec:sup_unsup} and \ref{ssec:loss}.
\par
Before discussing those aspects in detail, let us notice that the problem \eqref{eq:statIP} can be interpreted as a well-known task in statistics, namely the search for a \textit{point estimator} of the random variable $x$ given $y$ (see e.g. \cite{lehmann2006theory}). However, in contrast with the classical statistical learning scenario, we are assuming to know the (forward) model linking $x$ and $y$. The task of $R$ in \eqref{eq:statIP} is not to infer, or discover, a new model, but rather to provide a stable approximation of the ill-posed operator $F^{-1}$, based on statistical techniques. For this reason, problem \eqref{eq:statIP} is also sometimes addressed as a \textit{learned regularization} approach for the inverse problem \eqref{eq:IP}. The term \textit{regularization} is reminiscent of the one used in Section \ref{ssec:deterministic}, although the 
verification of the desired analytical properties of a regularizer (i.e., stability and convergence as $\delta$ is reduced) are not straightforward to be verified. For a bright overview of stability and convergence results in learned regularization, the reader is referred to \cite{mukherjee2022learned}. 
\par
In recent years, the field of learned reconstruction for inverse problems has flourished, leading to a large variety of techniques, often yielding striking numerical results in applications, and sometimes providing theoretical guarantees - or, at least, justifications. A limited review of data-driven approaches is reported in Section \ref{sec:learnedreg_review}; for a broader view on the subject, the reader is referred, e.g. to \cite{arridge2019} and \cite{mukherjee2022learned}.

\subsection{Other statistical approaches to inverse problems}
\label{ssec:otherstat}

The application of statistical models and methods in inverse problems has been extensively discussed and developed throughout the recent decades, in multiple distinct directions. It is therefore important to draw some connection between learned reconstruction and other selected approaches, to highlight common features and relevant differences.
\par
The first scenario is the one in which the noise $\varepsilon$ in \eqref{eq:IP} is assumed to be stochastic, and modeled as a random variable. From a frequentist point of view, a deterministic inverse problem as \eqref{eq:IP} might be interpreted as a worst-case scenario of a statistical one: nevertheless, the assumption $\| \varepsilon \|_Y \leq \delta$ allows to comprehend only bounded random variables, and excludes much more relevant noise models such as Gaussian white noise. For this reason, the study of large statistical noise has been the object of extensive studies, drawing techniques and tools from statistical estimation: see for example \cite{egger2008regularization, mathe2011regularization, hohage2016inverse, burger2018large}. The most important element that distinguishes these approaches from data-driven methods is the absence of a stochastic interpretation of the variable $x$.
\par
On the same line, but with a slightly different perspective, is the field of Statistical Inverse Learning problems, see in particular \cite{blanchard2018optimal} or \cite{rastogi2020convergence,bubba2023convex}. In such a scenario, the measurement $y$ is associated with a fixed, unknown, ground truth $x$, but is assumed to be randomly sub-sampled. This outlines an immediate connection with (supervised) statistical learning, since the data of the problem is represented by a set of point-wise measurements $y(u_n)$, associated with random evaluation points $u_n$: an inverse flavor is added to the problem since the goal is not to reconstruct $y$ itself, but rather $x$. 
\par
The most relevant comparison is nevertheless provided by Bayesian inverse problems \cite{kaipio2006statistical,calvetti2018inverse}. In this setup, $x$ and $\varepsilon$ are modeled as random variables, and the forward model linking them to $y$, as well as the distribution of $\varepsilon$, are known. Moreover, the probability distribution of $x$ is assumed to be known and is usually referred to as the \textit{prior} distribution. The goal of Bayesian inverse problems is to retrieve, from this information, the \textit{posterior} probability distribution, i.e., $\rho_{x|y}$: this can be achieved using Bayes' theorem. 
The resulting inverse problem is well-posed in the space of probability distribution: the prior $\rho_x$ plays a regularizing role, and one might rigorously address the topics of stability, convergence, and characterization of the posterior distribution, see \cite{evans2002inverse,stuart-2010,dashti2013bayesian} and \cite[Section 3]{arridge2019}. The target of Bayesian inverse problems is thus different from the one of \eqref{eq:statIP}: the posterior distribution is a more general tool than a learned operator $R$, and it can indeed be employed to compute point estimators, such as the conditional mean, or the mode (known as Maximum A Posteriori estimator, MAP). The most important difference between the two approaches resides, nevertheless, in the role of the distribution of $x$: in learned reconstruction methods, such a prior is a fixed, \textit{objective} item, encoding the true statistical properties of the ground truth, and is partially known, typically through a training sample. In Bayesian inverse problems, instead, such a distribution represents a \textit{subjective} prior, i.e., it reflects the belief of the solver, the features and properties the ground truth is expected to possess. Besides the interesting philosophical implications of this discussion, the difference is substantial. In Bayesian inverse problems, crafting a meaningful prior is a task of primary importance: more complicated distributions (based on complicated random variables, mixtures, or hyperparameters) might encode deeper features of the ground truth, such as smoothness, sparsity, or more. The choice of the prior might be inspired by available data but is not data-driven, similar to the choice of $\Psi$ in variational regularization. Conversely, in learned reconstruction, the prior $\rho_x$ cannot be freely handled by the user: it is fixed, though only partially known, and cannot be tailored to the solver's belief or for theoretical and algorithmic purposes. A new data-driven technique is not developed by designing a new prior distribution, but rather by providing a different (parametric) class of operators to approximate $F^{-1}$. 
Despite these distinctions, fruitful connections between statistical learning and Bayesian inverse problems already exist: for example, employing Plug and Play denoisers or generative models (described in Section \ref{sec:learnedreg_review}) to design data-driven priors (see \cite{laumont2022bayesian} and \cite{holden2022bayesian}, respectively).
\par
Finally, learned regularization of inverse problems \eqref{eq:statIP} can be interpreted in the field of \textit{operator learning}, which extends the original setup of statistical learning to (possibly) infinite-dimensional function spaces. In the broad literature of this field, remarkable recent contributions have shown interest in the topic of learning a suitable approximation also of inverse, ill-posed, maps: see, for example, \cite{kovachki2021neural} and \cite{mollenhauer2022learning}.

\section{
Learned reconstruction methods: a statistical perspective
}
\label{sec:framework}

Having introduced a general formulation for data-driven (or learned reconstruction) techniques for inverse problems in \eqref{eq:statIP}, this section aims to provide additional details, from a statistical standpoint, to identify common traits and different features of each technique. In particular,
Section \ref{ssec:sup_unsup} classifies learned methods based on the assumed partial information available on $\rho$, and Section \ref{ssec:loss} discusses one of the main choices to evaluate the quality of reconstruction operators: the expected loss. A common statistical trait is discussed instead in Section \ref{ssec:implicit}: all the techniques presented in the rest of the paper can be interpreted as parametric estimators, and relate to the concept of implicit regularization.

\subsection{Supervised and unsupervised learning}
\label{ssec:sup_unsup}
As formulated in \eqref{eq:statIP}, the problem of learning an estimator $R$ relies on partial knowledge of the joint probability distribution $\rho$ of the random vector $(x,y) \in X\times Y$. This reflects many possible scenarios, a complete classification of which is way beyond the scope of the present discussion. Let us for simplicity consider two possible approaches, although the main focus of this work will be on the first one.
\paragraph*{Supervised learning} 
Assume that the analytical expression of $\rho$ is unknown, but we have access to a sample
    \begin{equation}
        \label{eq:sup}
    \begin{gathered}
    \big\{(x_j, y_j)\big\}_{j=1}^m; \quad (x_j, y_j) \textit{ i.i.d.} \sim \rho,\\
    \textit{i.e., } \quad x_j \textit{ i.i.d.} \sim \rho_x, \quad \varepsilon_j \textit{ i.i.d.} \sim \rho_\varepsilon, \quad y_j = F(x_j) + \varepsilon_j
    \end{gathered}
    \end{equation}
    The pairs $(x_j,y_j)$ are often referred to as labeled couples, a name derived from classification problems in statistical learning (even though the current one is closer to a regression problem). The sample $\big\{(x_j, y_j)\big\}_{j=1}^m$ can be denoted as the \textit{training set}, and the quantity $m \in \N$ as the \textit{sample size}.
    \par
    Despite it can seem quite prohibitive to obtain them, supervised datasets are available for many inverse problems of interest. On the one hand, if the forward model $F$ is physically reliable and a numerical simulator is available, such a dataset can be generated \textit{synthetically}, by crafting some ground-truth objects, applying $F$, and corrupting the results with artificial noise. Consider, for example, an imaging inverse problem such as Computed Tomography: one might obtain a collection of images $x_j$ by generating realistic shapes or relying on high-fidelity reconstructions, even from different imaging modalities, applying the forward operator (which is well-known, according to the geometry of the device used for CT scans), and corrupting the obtained measurements with a suitable error model. Moreover, thanks to the great interest of several research teams in the task of supervised learning for inverse problems, some extremely useful training samples of real data have started to be generated and publicly released, such as \cite{kiss_2023_8017653} for CT reconstruction.
    
    \paragraph*{Unsupervised learning} 
    This denomination is often reserved for techniques that rely on datasets of the form
    \[ 
    \big\{y_j\big\}_{j=1}^m; \quad y_j \textit{ i.i.d.} \sim \rho_y
    \]
    (being $\rho_y$ the marginal distribution of $y$ derived from $\rho$ in \eqref{eq:rho}), thus consisting of non-matched measurements, or outputs of the forward model. We nevertheless classify as unsupervised techniques also the ones that employ datasets of the form
    \begin{equation}
     \label{eq:unsup}
    \big\{x_j\big\}_{j=1}^m; \quad x_j \textit{ i.i.d.} \sim \rho_x,
    \end{equation}
    thus also avoiding the use of labelled data. In particular, this second type of sample, although it does not carry information on the forward model $F$, can still successfully be employed to extract prior information on the ground truth $x$.
    \par
    Unsupervised datasets of signals $y_j$ are available for most inverse problems, by simply repeating the measurements of the forward models for multiple unknowns. Unsupervised datasets of ground truths can either be synthetically generated, either derived from suitably crafted datasets as \cite{kiss_2023_8017653}, or even obtained via high-fidelity reconstruction methods, applied on oversampled data (when available). Notice moreover that an unsupervised dataset of ground truths is independent of the considered inverse problem, and can be shared from one to another (e.g. MRI reconstructions of lung images can also be employed for CT or EIT).
\par
Although some of the techniques reported in Section \ref{sec:learnedreg_review} are unsupervised, the theoretical results of Section \ref{sec:main} are specific for \textbf{supervised} learning methods. 

\subsection{Loss, expected loss and empirical risk}
\label{ssec:loss}
The final clarification required by the expression \eqref{eq:statIP} is the metric in which the approximation of $x$ by $R(y)$ is measured. Let us first introduce the \textit{loss} function
\begin{equation}
    \ell \colon X \times Y \times \mathcal{M}(Y,X) \rightarrow \R,
    \qquad \text{s.t.} \quad \ell(x,y;R) = \tilde{\ell}(R(y),x),
    \label{eq:loss}
\end{equation}
being $\tilde{\ell}$ a metric in $X$, and $\mathcal{M}(Y,X)$ the space of measurable functions from $Y$ to $X$. As a main example, consider the \textit{quadratic loss}:
\begin{equation}
\tilde{\ell}(x,x') = \frac{1}{2} \| x- x'\|_X^2; \qquad \ell(x,y;R) = \frac{1}{2}\| R(y) -x\|_X^2.
    \label{eq:quad_loss}
\end{equation}
For a fixed estimator $R$, since $x$ and $y$ are random variables, the loss $\ell(x,y;R)$ is a real-valued random variable. This motivates the introduction of aggregate information to quantify the quality of the reconstruction $R(x)$: in particular, the expectation 
\begin{equation} \label{eq:expected_loss}
    L(R) = \E_{(x,y)\sim \rho} [\ell(x,y;R)]
\end{equation}
is defined as the \textit{expected loss}. In the case of the quadratic loss \eqref{eq:quad_loss}, $L$ is also known as the Mean Squared Error (MSE) of the estimator $R$. As it will be formalized in Section \ref{sec:errors}, the best estimator is the one that minimizes $L(R)$ among all possible choices of $R$.
\par
Clearly, evaluating $L(R)$ requires knowing the joint distribution $\rho$, which is only partially available according to \eqref{eq:statIP}: this would lead to a minimization problem in which the cost functional is not fully computable. In the context of supervised learning, though, a proxy of $L$ is available, obtained by simply replacing the expected value with a sample average. The quantity
\begin{equation}
\hat{L}(R) = \frac{1}{m} \sum_{j=1}^m \ell(x_j,y_j;R)
    \label{eq:empirical}
\end{equation}
can be referred to as \textit{empirical risk}. 
Notice carefully that the empirical risk should not be seen as a deterministic quantity (as if the sample \eqref{eq:sup} was a \textit{fixed} realization of the variables $x$ and $y$), but rather as a random variable itself, depending on the identically distributed random variables $x_j$ and $y_j$. The i.i.d. assumption in \eqref{eq:sup} immediately implies that $\E[\hat{L}(R)] = L(R)$: by the law of large numbers, we moreover expect that $\hat{L}(R)$ converges to $L(R)$ as $m$ grows, although quantitative concentration estimates are derived in Section \ref{sec:main} under additional assumptions on the random variable $\ell(x,y;R)$.
\par
The theoretical analysis performed in this work only focuses on the expected loss and empirical risk minimization. One remarkable alternative, which is not discussed here, is represented by \textit{adversarial} learning. In such a context, the metric $L$ is replaced by another functional, which allows the trained estimator to act as a critic, distinguishing a desirable probability distribution from another one. For a deep analysis of adversarial learning in inverse problems, and connections with optimal transport and manifold learning, the reader is referred to \cite{lunz2018adversarial}.

\subsection{Parametric estimation and implicit bias}
\label{ssec:implicit}
As noted in Section \ref{ssec:statisticIP}, a learned reconstruction operator $R$ as in \eqref{eq:statIP} can be interpreted as a point estimator of the random variable $x$ in terms of $y$. The minimal requirement on $R$ is that it is a measurable function from $Y$ to $X$, so that $R(y)$ is a random variable on $X$. The \textit{Bayes estimator} is defined by optimizing the expected loss among all the measurable functions in $\mathcal{M}(Y,X)$:
\begin{equation}
    \label{eq:Bayes_est}
R_\rho \in \argmin_{R \in \mathcal{M}(Y,X)} L(R) = \argmin_{R \in \mathcal{M}(Y,X)} \E_\rho[\ell(x,y;R)],
\end{equation}
although in general the existence and uniqueness of such a minimizer is not guaranteed. If $\ell$ is assumed to be the quadratic loss \eqref{eq:quad_loss}, the Bayes estimator is also known as the Minimum Mean Square Error (MMSE) estimator, and it is easy to show (see, e.g., \cite[Proposition 1]{cucker2002mathematical}) that it coincides with the conditional mean of $x$ given $y$:
\[
R_\rho(y) = \E_\rho[x|y] = \int_X x\ d \rho_{x|y},
\]
where $\rho_{x|y}$ is the conditional distribution of $x$ with respect to $y$.
\par
For different choices of the loss $\ell$, guaranteeing the existence and uniqueness of $R_\rho$, and providing a closed-form expression for it, can be much more problematic. Moreover, especially when dealing with the approximation of a nonlinear operator such as $F^{-1}$ in inverse problems, the operator $R_\rho$ might not be the ideal estimator. Consider, in particular, a linear inverse problem: in such a case, one obtains that the MMSE estimator coincides with the pseudoinverse $R_\rho = F^\dag$ of the forward operator. This entails that the Bayes estimates may inherit the ill-posedness of the inverse operator it is approximating - in particular, its instability.
\par
To overcome such a drawback, several remedies have been proposed in statistical inference and estimation, which can also be interpreted as a regularization of problem \eqref{eq:Bayes_est}. Several options, analogous to variational regularization, require the introduction of an additional term in \eqref{eq:Bayes_est} promoting desired properties of $R$. In the present discussion, we focus instead on techniques which restrict the space on which $L$ is minimized to a set $\mathcal{H} \subset \mathcal{C}(Y,X) \subset \mathcal{M}(Y,X)$, referred to as the \textit{hypothesis} class (here $\mathcal{C}(Y,X)$ denotes the space of the continuous operators from $Y$ to $X$). As it is evident in the case of the square loss (see again \cite[Proposition 1]{cucker2002mathematical}) for every measurable $R$,
\[ 
L(R) = \int_Y (R - R_\rho)^2 d\rho_y + L(R_\rho),
\]
which can be interpreted by saying that selecting an operator $R$ different from $R_\rho$ means to introduce a \textit{bias} in the estimation of $x$. The main challenge of regularized estimation techniques is taking advantage of such a bias to successfully balance stability and accuracy. Whenever the bias is obtained by manipulating the ambient space of the optimization problem, it is usually referred to as \textit{implicit}. In particular, it is most common that the class of functions $\mathcal{H}$ is described utilizing a set of parameters, collected in a vector $\theta \in \Theta$:
\begin{equation}
    \label{eq:parametric_H}
    \mathcal{H} = \{R_\theta\colon Y \rightarrow X \text{ s.t. } \theta \in \Theta \}.
\end{equation}
In statistics, operators of the form \eqref{eq:parametric_H} are denoted as \textit{parametric estimators}: in this scenario, the minimization on $\mathcal{H}$ is translated into a minimization on $\Theta$.
\par 
The rest of the discussion is mainly focused on \textbf{parametric} estimation: in particular, Section \ref{sec:learnedreg_review} shows through a limited review that many learned reconstruction techniques can be interpreted in such a framework, and are associated with different choices of (parametric) hypothesis classes.

\section{An essential literature review}
\label{sec:learnedreg_review}

Describing a family of operators using parametric maps is a task of broad importance, beyond the scope of learned regularization for inverse problems, and associated with a larger class of methods than empirical risk minimization or supervised learning. Let us now go through a list of approaches, for simplicity in the case of a linear inverse problem $y=Ax+\varepsilon$ (although many of them are suited also for nonlinear problems), providing a fairly broad picture of the possible strategies and their motivations. As already pointed out, the proposed review is not exhaustive and is only intended to provide a general picture.
\par
Most of the proposed techniques are based on \textbf{neural networks}, either fully connected, convolutional (CNN), residual (RNN), or with encoder-decoder architectures. This choice is usually motivated by the remarkable expressivity of such parametric functions (namely, the property of approximating a large family of operators) and their efficient numerical optimization with respect to the parameters, usually by means of stochastic optimization and automatic differentiation.

\paragraph*{Non-variational methods}

We start the review with parametric operators which are not associated with a minimization problem. 
The most straightforward paradigm consists in modeling $R_\theta$ directly as a neural network, e.g. a convolutional one (CNN), a residual one, or an encoder-decoder structure such as a U-Net (a review of similar techniques can be found in \cite{wang2016perspective,mccann2017convolutional,lucas2018using,arridge2019}). It is nevertheless well-established that, especially in the context of ill-posed inverse problems, such methods are prone to instability, and might show artifacts and hallucinations, especially when the datum is perturbed with adversarial noise (see \cite{antun2020instabilities}).
An important related approach is the one of \textit{learned post-processing}:
\[
R_\theta = C_\theta \circ A^\dag, 
\]
where $A^\dag\colon Y \rightarrow X$ is the pseudo-inverse of the forward operator or a fixed model-driven regularizer of the inverse problem. The map $C_\theta\colon X \rightarrow X$ is instead a learned post-processing operator, trained to remove noise and artifacts, which might appear in the reconstructions carried out by $A^\dag$. Such methods can be further distinguished in terms of the architecture employed for $C_\theta$ (e.g., a U-Net \cite{jin2017deep} or a framelet-based deep network \cite{ye2018deep}). 
More sophisticated approaches rely on a deeper connection between learned and model-based regularization: for example, \cite{schwab2019deep} focuses on a careful treatment of the null-space of the forward operator, ensuring convergence properties of the reconstruction; \cite{bubba2019learning}, instead, combines shearlet-based regularization with a data-driven term, according to insights coming from microlocal analysis.

\paragraph*{Variational methods}

In this case, $R_\theta$ is defined as the minimum of a parametric functional, $J_\theta$. Regardless of the specific parametrization of $J_\theta$, this paradigm can be addressed as \textit{bilevel optimization}, as the following (nested) minimization problems are considered:
\[
\min_{\theta \in \Theta} L(R_\theta), \qquad R_\theta(y) = \argmin_{x \in X} \{J_\theta(x;y)\}.
\]
The inner problem represents the variational nature of the regularizer $R_\theta$, the outer problem simply recalls that the optimal operator is trained to minimize a loss function $L$, such as \eqref{eq:expected_loss}. A pioneering work in bilevel optimization for inverse problems is \cite{haber2003learning}.
\par
Variational regularization techniques can be distinguished according to the choice of the parametrization of $J_\theta$ (a short case study is reported here), but also in terms of the loss which is used to train them. As expressed in Section \ref{ssec:loss}, the interest of this study is mainly in supervised learning with expected loss. Alternative choices of $L$ lead, e.g., to adversarial learning, which can be fruitfully studied in a learned variational framework, see \cite{lunz2018adversarial,mukherjee2020learned}.
\par
One popular parametrization of $J_\theta$ is by means of a (convolutional) neural network $C_\theta$, as follows:
\[
J_\theta(x;y) = \frac{1}{2}\| Ax - y \|_Y^2 + \alpha \| C_\theta(x)\|_X^p, 
\]
leading to the paradigm of \textit{Network Tikhonov} (NETT) and its extensions and variants (see
\cite{li2020nett, obmann2021augmented, bianchi2023uniformly}). Another possibility exploits the expressive power of a neural network $G_\theta$ to define the range on which the minimization problem takes place: in a synthesis formulation, this can be seen as
\[
J_\theta(x;y) = \frac{1}{2}\| AG_\theta(x) - y \|_Y^2.
\]
This scenario is referred to as \textit{generative} models, which can be articulated in rather different versions (see \cite{rick2017one, habring2022generative, alberti2022continuous}).
The technique in \cite{aspri2020data}, which draws connections with regularization by projection, can be interpreted in a similar framework, by replacing the nonlinear operator $G_\theta$ with a suitable learned linear projection.
\par
For some simple choices of $J_\theta$, the minimizer $R_\theta$ can be provided in explicit form: this is the leading idea behind the search for the optimal (generalized) Tikhonov functional (\cite{alberti2021learning}), in which $\theta=(h,B)$, $h \in X$, $B\colon x \rightarrow X$ and
\[
J_\theta(x) = \frac{1}{2}\|Ax-y\|_{\Sigma_\varepsilon}^{2} + \| B(x-h) \|_X^2
\]
On a similar note (but also possibly outside the context of variational regularization), one may craft data-driven regularizers by learned filtering of the singular values of the forward map $A$, again deriving explicit formulas of $R_\theta$ (see \cite{kabri2022convergent}).

\paragraph*{Methods approximating variational ones}
The paradigm of learned variational regularization, or bilevel learning, is extremely appealing from a theoretical point of view, as it combines a remarkable expressive power with theoretical guarantees. Nevertheless, from a computational perspective, some choices of $J_\theta$ may be rather unpractical, especially when relying on first-order methods for the outer minimization problem, which requires differentiating $R_\theta$ (the solution of the inner problem) with respect to $\theta$. For this reason, it can be beneficial to define $R_\theta$ as an approximation of the solution of a variational problem, and in particular, as the one produced by an iterative minimization scheme, after a suitable number of iterations. If a total number of iterations $L$ is fixed, one gets
\[
R_\theta(y) = \big( \Lambda_{\theta_L} \circ \ldots \circ \Lambda_{\theta_2} \circ \Lambda_{\theta_1}\big)(y),
\]
where a (small) neural network $\Lambda_{\theta_l}$ is used to replace each iteration of an iterative scheme, and the parameters $\theta_l$ of each layer are collected in a global vector $\theta = (\theta_1,\ldots,\theta_{L})$. The resulting operator represents the outcome of $L$ iterations of an optimization algorithm, and it can be interpreted as a (deep) neural network, and trained end-to-end: this motivates the name of the technique, \textit{unrolling}. The seminal idea of unrolling can be found in \cite{gregor2010learning}; its application in inverse problems has yielded striking results \cite{adler2017solving,adler2018learned} for linear and non-linear inverse problems (see also \cite{mukherjee2021end} for an adversarial version). When designing an unrolling technique, it is possible to craft the layers $\Lambda_\theta$ so that they replicate some known iterative schemes, endowed with learnable features. To name a few possibilities, this can be used to learn the proximal operator of a data fidelity and a regularization terms as in \cite{adler2018learned}, to retrieve an optimal sparse representation of the unknown, either via a kernel and multi-resolution-based dictionary as in \cite{hammernik2018learning,kobler2020total} or via sparse linear operators as in \cite{de2022deep} (for some classes of operator-based inverse problems), to replicate a multi-scale iterative reconstruction as in \cite{hauptmann2020multi}, or to add a suitable correction of the forward map as in \cite{bubba2021deep}.

\paragraph*{Fixed-point methods}
The original motivation of unrolling is to approximate variational methods: nevertheless, since in unrolled schemes the parameter $\theta_l$ is allowed to change throughout the iterations/layers, in most cases it is difficult to associate the resulting network with the (approximate) minimization of a functional. An alternative technique, which avoids similar shortcomings, is based on a different consideration: namely, that minimizers of functionals from $X$ to $\R$ can often be written as fixed points of suitable operators from $X$ to $X$.
For this reason, an alternative technique has been proposed, based on the \textit{Deep Equilibrium} paradigm (see \cite{bai2019deep} and its declination for imaging problems in \cite{gilton2021deep}). In this case, $R_\theta$ is defined as a fixed point of a map $\phi_\theta$: 
\[
R_\theta(y) = \operatorname{Fix}(\phi_\theta(\cdot;y)), \quad\text{i.e.,} \quad R_\theta(y) = p_\theta\ \text{ s.t. }\ p_\theta = \phi_\theta(p_\theta;y).
\]
Again, the map $\phi_\theta$ can be modeled to resemble operators with a theoretical motivation from a variational perspective (e.g., proximal operators), with learnable features. In this case, one can optimize in $\theta$ without the need for unrolling: this is possible through implicit differentiation of the fixed point equation, which can be effectively combined with automatic differentiation schemes. For this reason, deep equilibrium schemes are also referred to as \textit{infinite-depth} networks.
\par
Finally, some techniques related to the so-called \textit{Plug-and-Play} approach can be analyzed in this same framework. In this case, 
\[
R_\theta(y) = p_\theta \quad \text{ s.t. } \quad p_\theta = D_\theta\big(p_\theta - \mu A^*(Ap_\theta-y)\big),
\]
which can be seen as a specific version of a fixed-point method. In contrast with deep equilibrium, though, here the learnable functional $D_\theta$ is pre-trained separately (using an unsupervised dataset) and can be later used to compute a fixed point, e.g., by means of a convergent iterative algorithm. This is motivated by the intuition that the proposed scheme resembles a forward-backward algorithm, and $D_\theta$ plays the role of a proximal map, which can be successfully replaced by a denoiser. The theoretical backbone of this technique is given by Tweedie's formula, linking the optimal denoiser with the prior distribution of the unknown $x$. For this reason, Plug-and-Play denoisers are not trained end-to-end to minimize the expected loss \eqref{eq:expected_loss}, and these techniques are actually out of the scope of the current analysis. The interested reader is however referred to \cite{venkatakrishnan2013plug,ryu2019plug,hurault2021gradient}.

\section{Learning, from the errors}
\label{sec:errors}

This section aims to introduce and clearly formulate the theoretical goals motivating the main results of the paper, contained in Section \ref{sec:main}, also connecting them with the statistical learning literature (see, for a comparison \cite{cucker2002mathematical} and \cite[Chapter 2]{shalev2014}). In particular, among the many theoretical aspects associated with learned reconstruction methods \eqref{eq:statIP}, this dissertation focuses on \textit{generalization} estimates, quantifying the ability of the learned operator to provide good estimates also when tested on inputs different from its training sample. A more detailed definition of such an error, together with the description of the components into which it is usually split, is reported in \ref{ssec:targets}, whereas Section \ref{ssec:sample} presents a focus on one of these components, which is the main subject of the theoretical investigation of this work.

\subsection{Targets and errors in statistical learning}
\label{ssec:targets}
As already discussed in Section \ref{ssec:implicit}, the goal of data-driven methods \eqref{eq:statIP} is usually not to retrieve the Bayes estimator $R_\rho$ as in \eqref{eq:Bayes_est}, but rather
\begin{equation}
\label{eq:optimal_target}
    R^\star = R_{\theta^\star}, \qquad \theta^\star \in \argmin_{\theta \in \Theta} L(\theta) = \argmin_{\theta \in \Theta} \E_{\rho}[\ell(x,y;\theta)]
\end{equation}
where, with a little abuse of notation, $L$ and $\ell$ are expressed as functions of $\theta$ instead of $R$. The object $R^\star$ (and, sometimes, its parameter $\theta^\star$) is referred to as the \textit{optimal target} in $\mathcal{H}$ (or in $\Theta$). Of course, the minimization of $L$ can only be performed via the full knowledge of the probability distribution $\rho$, which may not be available in \eqref{eq:statIP}.
\par
In a supervised learning scenario, the information provided by a training set $\{(x_j,y_j)\}_{j=1}^m$ as in \eqref{eq:sup} can be used to approximate the expected loss via the empirical risk $\hat{L}$ as in \eqref{eq:empirical}: minimizing $\hat{L}$ leads to the so-called \textit{empirical target}
\begin{equation}
\hat{R} = R_{\hat{\theta}}, \qquad \hat{\theta} \in \argmin_{\theta \in \Theta} \hat{L}(\theta) = \argmin_{\theta \in \Theta} \frac{1}{m} \sum_{j=1}^m \ell(x_j,y_j;\theta).
    \label{eq:empirical_target}
\end{equation}
Remembering that $\hat{L}(\theta)$ is a random variable, its minimizers $\hat{\theta}$ should also be considered as random variables, depending on the randomness of the sample $(x_j,y_j)$.
A useful tool to prove the measurability of minimizers of random functionals is Aumann's selection principle, see for example \cite[Lemma A.3.18]{steinwart2008support} or \cite[Lemma 6.23]{steinwart2008support}.
\par
The existence (and, possibly, the uniqueness) both of the empirical and the optimal target can be ensured by suitably selecting the loss function $\ell$, the parametrization of the functionals in $\mathcal{H}$ by $\theta$ and by the properties of $\Theta$: this task is up to anyone who first proposes a new data-driven strategy. For the purpose of this discussion, let us assume that a minimizer exists both for $L$ and $\hat{L}$ (although maybe not unique): actually, one might do without this assumption by substituting, in the rest of the section, the expression $L(\theta^\star)$ by $\inf_{\theta} L(\theta)$ and the expression $\hat{L}(\hat{\theta})$ by $\inf_{\theta} \hat{L}(\theta)$.
\par
Finally, when minimizing $\hat{L}$ over $\Theta$, several techniques from convex and non-convex analysis might be employed, such as stochastic gradient descent, alternating directions, and zero-order methods. The outcome of these algorithms may, in general, be different from the empirical target: for this reason, let us denote it as $\tilde{R} = R_{\tilde{\theta}}$, the \textit{learning outcome}.
\par 
In order to quantify the quality of the recovered estimator $\tilde{R}$, the most natural metric is represented by the
\begin{equation}
\textit{generalization error}: \quad L(\Tilde{R}).
    \label{eq:generalization}
\end{equation}
The name reflects the fact that the estimator $\Tilde{\theta}$, which has been learned thanks to the training set \eqref{eq:sup}, is supposed to perform well on the whole sample space $X \times Y$, when $(x,y)$ is distributed as $\rho$. This quantity is traditionally split into different components (see \cite{cucker2002mathematical}):
\begin{itemize}
    \item a first part which solely depends on the optimization strategy employed to solve \eqref{eq:empirical_target}
    \begin{equation}
        \textit{optimization error}: \quad L(\tilde{R}) - L(\hat{R})
    \end{equation}
    \item a second part showing the dependence of the empirical target from the training sample \eqref{eq:sup}
    \begin{equation}
        \textit{sample error}: \quad L(\hat{R}) - L(R^\star)
    \end{equation}
    \item a third part which reveals the effect of restricting the choice of $R$ from $\mathcal{M}(Y,X)$ to the hypothesis class $\mathcal{H}$:
    \begin{equation}
        \textit{approximation error}: \quad L(R^\star) - L(R_\rho)
    \end{equation}
    \item a last term quantifying the inherent uncertainty associated with the estimation problem, which cannot be captured by the Bayes estimator
    \begin{equation}
        \textit{irreducible error}: \quad L(R_\rho).
    \end{equation}  
\end{itemize}
The focus of this work is essentially on \textbf{sample error estimates}, in order to prove its decay as $m$, the sample size, increases. In particular, recalling that $\hat{\theta}$ is a random variable, we want to verify such a decay with large probability, or in expectation. For this reason, the goal of Section \ref{sec:main} is the derivation of estimates of the form
\begin{equation}
\Prob_{(x_j,y_j)\sim \rho^m} \big[|L(\hat{\theta})-L(\theta^\star)| \leq \varrho \big] \geq 1-\eta \qquad \text{for } m\geq m(\eta,\varrho) 
    \label{eq:PAC_prob}
\end{equation}
or (expressing $\varrho = \varrho(m,\eta)$ and integrating the tail bound) as
\begin{equation}
\E_{(x_j,y_j)\sim \rho^m} \big[|L(\hat{\theta})-L(\theta^\star)|\big] \leq \bar{\varrho}(m),
    \label{eq:PAC_exp}
\end{equation}
in accordance with the paradigm of Probably Approximately Correct (PAC) learning, see e.g. \cite[Chapter 3]{shalev2014}.

\subsection{Sample error: avoiding the curse of dimensionality}
\label{ssec:sample}
It is easy to notice that sample error bounds are strictly related to concentration estimates. Indeed, a standard way to tackle those bounds is to proceed as follows:
\[
L(\hat{\theta})-L(\theta^\star) = L(\hat{\theta})- \hat{L}(\hat{\theta}) + \hat{L}(\hat{\theta}) - \hat{L}(\theta^\star) + \hat{L}(\theta^\star) - L(\theta^\star),
\]
and since the central difference is non-positive by definition of $\hat{\theta}$,
\begin{equation}
    \label{eq:representativeness}
    |L(\hat{\theta})-L(\theta^\star)| \leq |L(\hat{\theta})- \hat{L}(\hat{\theta})| + |\hat{L}(\theta^\star) - L(\theta^\star)| \leq 2 \sup_{\theta \in \Theta} |\hat{L}(\theta) - L(\theta)|
\end{equation}
The last quantity can be defined as the \textit{representativeness} of the training set $\{(x_j,y_j)\}_{j=1}^m$ (see \cite[Chapter 26]{shalev2014}).
\par
Let us fix a parameter $\theta$: as already observed, $\hat{L}(\theta)$ and $L(\theta)$ are, respectively, the empirical average and the expected value of the variables $\ell(x_j,y_j;\theta)$. Thus, inequalities of the form 
\begin{equation}
\label{eq:concentration}
    \Prob_{\rho^m}\left[|\hat{L}(\theta) - L(\theta)| > \varrho\right] \leq \eta(\varrho,m)
\end{equation}
are quantitative versions of the law of large numbers, and are referred to as \textit{concentration} estimates, and can be deduced under rather general assumptions on the probability distribution of $\ell(x,y;\theta)$. 
For example, if $\ell(x,y;\theta)$ is either a bounded random variable or a sub-Gaussian one (see Section \ref{sec:main} for a detailed definition), Hoeffding's inequality ensures that, for some positive constant $K>0$,
\begin{equation}
\label{eq:hoeffding}
    \eta(\varrho,m) = 2 \exp\left(-\frac{2\varrho^2}{m K^2}\right).
\end{equation}
How to combine concentration estimates for each single parameter into a bound for the representativeness, thus on the sample error? The simplest scenario occurs when the parameter space $\Theta$ consists of a finite number of elements, i.e., its cardinality $|\Theta| \in \N$. If that is the case, by the union bound,
\[
\begin{aligned}
\Prob_{\rho^m}\left[  \sup_{\theta \in \Theta} |\hat{L}(\theta) - L(\theta)|\leq \frac{\varrho}{2} \right] &\geq 1 - \sum_{\theta \in \Theta} \Prob_{\rho^m}\left[ |\hat{L}(\theta) - L(\theta)| > \frac{\varrho}{2} \right] \\
&\geq 1 - |\Theta|\eta\!\left( \frac{\varrho}{2},m \right).
\end{aligned}
\]
If the loss is either bounded or sub-Gaussian, via \eqref{eq:hoeffding}
\[
\Prob_{\rho^m}\left[  \sup_{\theta \in \Theta} |\hat{L}(\theta) - L(\theta)|\leq \frac{\varrho}{2} \right] \geq 1 - 2|\Theta| \exp\left(-\frac{\varrho^2}{2mK^2}\right)
\]
This leads to the classical PAC learning estimate for finite classes (see, e.g. \cite[Chapter 4]{shalev2014}):
\begin{equation}
\label{eq:finite_set}
    \begin{gathered}
\Prob_{\rho^m} [|L(\hat{\theta})-L(\theta^\star)| \leq \varrho] \geq 1-\eta \quad \text{for } m\geq c \frac{\log(|\Theta|/\eta)}{\varrho^2} 
\end{gathered}
\end{equation}
To generalize from finite to infinite parameter sets $\Theta$, more refined techniques are required. One effective strategy is the one based on the concept of \textit{shattening} and with the tool known as Vapnik-Chervonenkis (VC) dimension, see e.g. \cite[Chapter 6]{shalev2014}. Despite the definition being rather involved, it is quite often the case that, if $\Theta$ is a subset of a vector space of dimension $d$, the VC-dimension of $\Theta$ coincides with $d$. In an informal formulation (for a precise statement see, e.g. \cite[Theorem 6.8]{shalev2014}), if the VC-dimension of $\Theta$ is $d$, it holds
\begin{equation}
\label{eq:finite_dim}
    \begin{gathered}
\Prob_{\rho^m} [|L(\hat{\theta})-L(\theta^\star)| \leq \varrho] \geq 1-\eta \quad \text{for } m\geq c \frac{d + \log(1/\eta)}{\varrho^2}
\end{gathered}
\end{equation}
Unfortunately, those bounds deteriorate as the number of parameters grows; when dealing with inverse problems, since the spaces $X$ and $Y$ are often infinite-dimensional, it might be useful to consider a parametrization $R_\theta$ by means of a parameter $\theta$ also living in an infinite-dimensional space (e.g., a functional space). As it is clear, it is not possible to adapt \eqref{eq:finite_dim} to an infinite-dimensional setup. In Section \ref{sec:main} we will explore a rather general 
approach, based on the related concepts of \textit{covering} and \textit{complexity}, which are naturally formulated for compact subsets of (possibly infinite-dimensional) metric spaces. Convergence rates analogous to the finite-dimensional setup \eqref{eq:finite_dim} can be achieved employing a more advanced technique, known as \textit{chaining}.

\section{Sample error estimates for learned reconstruction methods: main results}
\label{sec:main}

This section provides a rather general framework in which it is possible to prove sample error estimates of the form \eqref{eq:PAC_exp}, under rather general assumptions on the parameters class $\Theta$ (including infinite-dimensional sets) and on the random distribution of the variables $X$ and $y$ (including unbounded random variables). 
In particular, Section \ref{ssec:subs} introduces the families of probability distributions which will be the object of the rest of the analysis; Section \ref{ssec:covering} shows how to obtain convergence rates by means of covering numbers and Section \ref{ssec:chaining}, under slightly more restrictive assumptions, provides better rates through a modern technique known as \textit{chaining}. Sample error estimates analogous to the ones in Section \ref{ssec:covering} (although for a more restricted class of losses) can be found in \cite{cucker2002mathematical} and, for learned regularization, in \cite{alberti2021learning}: therefore, the discussion in the following section may help in improving the convergence rates which are proved therein. \\
The main techniques and results are drawn from \cite{vershynin2018high} and \cite{wainwright2019high}; for the sake of completeness, a short verification of unproven claims and extensions of those results is provided, when needed. 

\subsection{Sub-Gaussian, sub-exponential, and $q-$Orlicz random variables}
\label{ssec:subs}
As depicted in Section \ref{ssec:sample}, at the core of sample error bounds is the use of concentration estimates. In particular, according to \eqref{eq:representativeness}, the main goal consists in controlling, uniformly in $\theta$, the term
\[
|\hat{L}(\theta)- L(\theta) | = \left| \frac{1}{m} \sum_{j=1}^m \big(\ell(x_j,y_j;\theta) - \E[\ell(x,y;\theta)]\big)\right|.
\]
Let us introduce the notation:
\[
    Z_{\theta,j} = \ell(x_j,y_j;\theta) - \E[\ell(x,y;\theta)]; \qquad Z_\theta = \frac{1}{m} \sum_{j=1}^m Z_{\theta,j} = \hat{L}(\theta) - L(\theta);
\]
then, the goal is to bound the empirical average $Z_\theta$ of the zero-mean random variables $Z_{\theta,j}$, i.e., to provide a quantitative version of the law of large numbers. Notice that, whatever the choice of $X,Y,\Theta$, the random variables $Z_\theta$ are real-valued. 
\par
In statistical learning, it is quite often assumed that $Z_\theta$ are \textit{bounded} random variables, i.e., whose support is contained within an interval. For example, in classification cases, the output space $X$ is a finite set, thus a bounded loss function $\ell(x,y;\theta)=\Tilde{\ell}(R_\theta(y),x)$ surely takes values in a bounded subset of $\R$. In the case of inverse problems, though, this is quite a restrictive assumption: despite the ground truth variable $x$ might be subject to state constraints (e.g., non-negativity, box constraints,...), the random noise $\varepsilon$ is usually modeled as a Gaussian random variable, thus unbounded. As a consequence, the variable $y = F(x)+\varepsilon$ is unbounded, and most likely also $\ell(x,y;\theta)$, therefore $Z_\theta$. This motivates the need for a broader class of random variables which could still imply the desired estimates. 
\par
In particular, we consider \textit{sub-Gaussian} and \textit{sub-exponential} random variables, the most prominent examples of $q-$Orlicz variables. For $q \in [1,2]$, let $\psi_q(x) = \exp(x^q)-1$: it is possible to prove that 
\begin{equation}
    \label{eq:q-Orlicz-norm}
    \| W \|_{\psi_q} = \inf\{ t>0 : \E[\exp(|W|^q/t^q)] \leq 2\}
\end{equation}
is a norm on the space of random variables on $X$, and induces a Banach space structure, known as the $q-$Orlicz space. The cases $q\in\{1,2\}$ are extensively studied, e.g., in \cite[Sections 2.5, 2.7]{vershynin2018high}, also deducing equivalent formulations of \eqref{eq:q-Orlicz-norm}. For the sake of completeness, here is a small resume of the ones that are more relevant for later purposes. 

\begin{itemize}
    \item \textbf{Sub-Gaussian} random variables ($q=2$). By \cite[Proposition 2.5.2]{vershynin2018high}, the claim $\|W\|_{\psi_2} \leq K$ is equivalent to the following bounds on the tails and moments:
    \begin{align} 
    \Prob[|W|>t] &\leq 2 \exp(-t^2/K^2) \quad \forall t \geq 0 \label{eq:sub-G-tail} \\
    \E[|W|^p]^{1/p} &\leq C_1 K \sqrt{p} \label{eq:sub-G-mom}.
    \end{align}
    being $C_1>0$ an absolute constant. Condition \eqref{eq:sub-G-tail} motivates the employed denomination: the tails of sub-Gaussian random variables are bounded by the ones of a Gaussian one.
    Notice that both Gaussian and bounded real-valued random variables are sub-Gaussian, with norm equal to, respectively, their standard deviation or the length of their support.
    \item \textbf{Sub-exponential} random variables ($q=1$). Also in this case, by \cite[Proposition 2.7.1]{vershynin2018high}, $\|W\|_{\psi_1}\leq K$ can be equivalently reformulated as
    \begin{align} 
    \Prob[|W|>t] &\leq 2 \exp(-t/K) \quad \forall t \geq 0 \label{eq:sub-exp-tail} \\
    \E[|W|^p]^{1/p} &\leq C_1 K p \label{eq:sub-exp-mom}.
    \end{align}
    Again, \eqref{eq:sub-exp-tail} motivates the name: tails of sub-exponential random variables are below the ones of an exponential one. All sub-Gaussian random variables are sub-exponential and, moreover, the square of a sub-Gaussian random variable is sub-exponential (see \cite[Lemma 2.7.6]{vershynin2018high}): for this reason, this class is of great interest for regression problems on sub-Gaussian random variables $x,y$ using the square loss (see Section \ref{sec:examples}). Moreover, if $\E[W]=0$, the following equivalent condition holds on the moment-generating function of $W$:
    \begin{equation}
    \E[\exp(\lambda W)] \leq \exp(C_2^2K^2\lambda^2) \quad \forall \lambda: |\lambda| \leq \frac{1}{C_2 K_2}
        \label{eq:sub-exp-MGF}
    \end{equation}
\end{itemize}
Before delving into concentration estimates for sub-Gaussian and sub-exponential random variables, let us consider an important property of the empirical average for those classes of random variables.
\begin{proposition}
\label{prop:emp_aver}
Let $W_{j}$, with $j=1,\ldots,m$, be zero-mean i.i.d. $q-$Orlicz,  $q=\{1,2\}$, with $\|W_{j} \|_{\psi_q} = K$. Then, their empirical average $W = \frac{1}{m} \sum_j W_{j}$ is also $q-$Orlicz, and
\begin{equation}
\| W \|_{\psi_q} \leq \frac{K}{\sqrt{m}}.
    \label{eq:emp_aver}
\end{equation}
\end{proposition}
\begin{proof}
    The case of sub-Gaussian random variables is already discussed in \cite[Proposition 2.6.1]{vershynin2018high}. For $q=1$, using \eqref{eq:sub-exp-MGF} (as $W_j$ and $W$ are zero-mean) it holds
    \[
    \begin{aligned}
        \E\left[\exp\left( \lambda\sum_{j=1}^m W_j \right)\right]     & = \prod_{j=1}^m  \E[\exp(\lambda W_j)] 
        \\ & \leq \prod_{j=1}^m \exp(C_2^2\|w\|_{\psi_1}^2\lambda^2) \qquad \forall \lambda \text{ s.t. } |\lambda|\leq\frac{1}{C_2K},
        \\ &\leq \exp(m C_2^2K^2\lambda^2) \qquad \forall \lambda \text{ s.t. } |\lambda|\leq\frac{1}{\sqrt{m}C_2K},
    \end{aligned}
    \]
    where we used that $\frac{1}{\sqrt{m}C_2K} \leq \frac{1}{C_2K}$. We thus deduce that $\| W_1 + \ldots + W_m \|_{\psi_1} \leq \sqrt{m} K$; the thesis then follows by the homogeneity of the norm.
\end{proof}

It is now possible to formulate the first assumption which will be employed for the main results contained in the following sections:
\begin{assumption}  
\label{ass:sub}
For every $\theta \in \Theta$, let $\ell(x,y;\theta)$ be either a sub-Gaussian or a sub-exponential real-valued random variable, with $\| \ell(x,y;\theta)\|_{\psi_q} \leq K$ ($q=2$ or $q=1$, respectively).
\end{assumption}
Thanks to Proposition \ref{prop:emp_aver}, Assumption \ref{ass:sub} entails that, for every $\theta$, the excess risk $Z_\theta$ is either sub-Gaussian or sub-exponential, zero-mean, and with $\| Z_\theta \|_{\psi_q} \leq \frac{K}{\sqrt{m}}$.

\subsection{Sample error bounds based on covering}
\label{ssec:covering}

This section describes a strategy to bound $\displaystyle\sup_{\theta \in \Theta} |Z_\theta|$ based on Assumption \ref{ass:sub} and on two main pillars: the compactness of the set $\Theta$ and the stability of the operators $R_\theta$. Let us first clearly formulate these properties.
\begin{assumption}
\label{ass:compact}
Let $\Theta$ be a compact subset of a metric space with respect to a suitable metric $d$.
\end{assumption}
In particular, this implies that $\Theta$ is bounded with respect to $d$: let us denote by $D$ its diameter
\begin{equation}
D =\operatorname{diam}(\Theta) = \sup_{\theta,\theta' \in \Theta}{d(\theta,\theta')}.
    \label{eq:diam}
\end{equation}
In sample error analysis, the parameter class $\Theta$ is fixed, so via rescaling one might set $D=1$. For the sake of generality, let us keep track of $D$, and for later convenience let us also assume that $D \geq 1$.\\
Assumption \ref{ass:compact} also entails that, for every radius $r>0$, the \textit{covering numbers} of $\Theta$ with respect to $d$ are finite:
\begin{equation}
\mathcal{N}(\Theta,r) = \inf\big\{\#A; \quad A \subset \Theta \textit{ s.t. } \forall \theta \in \Theta, \exists a \in A: d(a,\theta) \leq r\big\}
    \label{eq:covering_numbers}
\end{equation}
As discussed in Section \ref{sec:examples}, examples of compact parameter spaces are easily constructed in a finite-dimensional context (such as balls or hyper-cubes) but also in infinite-dimensional spaces (leveraging e.g., compact embeddings of functional spaces).
Let us consider a $r-$covering $\Theta_N$ of $\Theta$, namely, a set of $N=\mathcal{N}(\Theta,r)$ elements as in the definition \eqref{eq:covering_numbers}. Then,
\[
\forall \theta \in \Theta,\ \exists \theta' \in \Theta_N: \quad |Z_\theta| \leq |Z_{\theta'}| + \big| |Z_\theta| - |Z_{\theta'}| \big|\ \text{ and } \ d(\theta,\theta') \leq r,
\]
and as a consequence
\begin{equation}
\sup_{\theta \in \Theta} |Z_\theta| \leq \sup_{\theta \in \Theta_N} |Z_\theta| + \sup_{\theta,\theta': d(\theta,\theta')\leq r} \big| |Z_\theta| - |Z_{\theta'}| \big|
    \label{eq:supremum}
\end{equation}
The first term in \eqref{eq:supremum} can be effectively controlled through the union bound or more refined techniques, since it involves a finite number of elements; for the second term, a stability argument is needed. 
\begin{assumption}
\label{ass:stab}
    Let $\ell(x,y;\theta)$ be locally H\"older stable with respect to $\theta$, uniformly in $x,y$, namely, $\exists r_0>0$ such that, if $d(\theta,\theta')\leq r_0$:
    \begin{equation}
    |\ell(x,y;\theta) - \ell(x,y;\theta')| \leq d(\theta,\theta')^\alpha \xi(x,y) \quad \rho-a.e.,
        \label{eq:stab}
    \end{equation}
    being $\alpha \in (0,1]$ and $\xi(x,y)$ a real-valued, positive random variable such that $\E[\xi(x,y)] \leq M_\ell$.
\end{assumption}
\begin{proposition}
\label{prop:covering}
Let $\Theta$ satisfy Assumption \ref{ass:compact} and $\ell$ satisfy Assumptions \ref{ass:stab} and \ref{ass:sub} with $q=1,2$. Then, for all $r \leq r_0$, 
\begin{equation}
\label{eq:covering}
    \E[L(\hat{\theta})-L(\theta^\star)] \leq C \frac{K}{\sqrt{m}} \log(\mathcal{N}(\Theta,r))^{1/q} + 2 M_\ell r^\alpha,
\end{equation}
being $C>0$ an absolute constant.
\end{proposition}
\begin{proof}
According to \eqref{eq:representativeness} and \eqref{eq:supremum}, it holds
\begin{equation}
\E[L(\hat{\theta})-L(\theta^\star)] \leq 2 \E\left[\sup_{\theta \in \Theta_N} |Z_\theta|\right] + 2 \E\left[\sup_{\theta,\theta': d(\theta,\theta')\leq r} \big| |Z_\theta| - |Z_{\theta'}| \big|\right].
    \label{eq:aux_exp}
\end{equation}
To estimate the first term we can take advantage of \cite[Exercise 2.19]{wainwright2019high} (for the sub-Gaussian case, see also \cite[Exercise 2.5.10]{vershynin2018high}): since $\| Z_\theta\|_{\psi_q} \leq \frac{K}{\sqrt{m}}$,
\begin{equation}
    \label{eq:sup_Orlicz}
    \E\left[\sup_{\theta \in \Theta_N} |Z_\theta|\right] \leq \frac{K}{\sqrt{m}}\psi_q^{-1}(N) \leq C \frac{K}{\sqrt{m}} (\log(N))^{1/q}. 
\end{equation}
Regarding the second term in \eqref{eq:aux_exp}, for each $\theta,\theta' \in \Theta$:
\[
\begin{aligned}
    \big| |Z_\theta| - |Z_{\theta'}| \big| &= \big| |\hat{L}(\theta)-L(\theta)| - |\hat{L}(\theta')-L(\theta')| \big|\\
& \leq \big| \hat{L}(\theta)-\hat{L}(\theta') - L(\theta)+L(\theta')  \big| \\
&\leq \frac{1}{m} \sum_{j=1}^m |\ell(x_j,y_j;\theta)-\ell(x_j,y_j;\theta')| + \E[|\ell(x,y;\theta)-\ell(x,y;\theta')|] \\
&\leq d(\theta,\theta')^\alpha \left(\frac{1}{m} \sum_{j=1}^m \xi(x_j,y_j) + \E[\xi(x,y)]\right)
\end{aligned}
\]
Therefore, the following bound holds uniformly in $\theta,\theta'$
\[
\sup_{\theta,\theta': d(\theta,\theta')\leq r} \big| |Z_\theta| - |Z_{\theta'}| \big| \leq r^\alpha \left(\frac{1}{m} \sum_{j=1}^m \xi(x_j,y_j) + \E[\xi(x,y)]\right)
\]
and by taking the expectation one immediately concludes the thesis.
\end{proof}
Proposition \ref{prop:covering} provides a sample error bound in expectation: nevertheless, it can be translated into a bound in probability (analogous to \eqref{eq:finite_dim}), e.g., via Markov's inequality or with more advanced tools such as \cite[Lemma 5.37]{wainwright2019high}. \\
Finally, the optimal choice of $r$ in \eqref{eq:covering} can be made only when an analytic expression of $\mathcal{N}(\Theta,r)$ is provided: some examples are discussed in Section \ref{sec:examples}.

\subsection{Sample error bounds based on chaining}
\label{ssec:chaining}

The technique outlined in the previous section shows that sample error estimates as \eqref{eq:PAC_exp} can be derived even for infinite-dimensional parameter classes. Nevertheless, it is possible to achieve faster rates in $m$, especially by a more careful treatment of the first term in \eqref{eq:supremum}. The idea of \textit{chaining} dates back to the pioneering work of Dudley \cite{dudley1967sizes}, and has been further developed by Talagrand, e.g. in \cite{talagrand2005generic}. The essential idea behind it is that term $\E[\sup_{\Theta_N}|Z_\theta|]$, which was previously estimated via the union bound, can also benefit from the stability assumption we introduced on the random variables. In particular, this is possible through the introduction of a specific iterative refinement of the covering of $\Theta$. A detailed discussion on the chaining procedure can be found in \cite[Chapter 5]{wainwright2019high} and \cite[Chapter 8]{vershynin2018high}.
\par
To apply this technique, a slightly more restrictive request must be introduced rather than Assumptions \ref{ass:sub} and \ref{ass:stab}.
\begin{assumption}
\label{ass:proc}
    Suppose $\ell(x,y;\theta)$ is H\"older stable with respect to $\theta$, uniformly in $x,y$. Moreover, for $q \in \{1,2\}$, let
    \begin{equation}
    \| \ell(x,y;\theta) - \ell(x,y;\theta') \|_{\psi_q}  \leq K_\ell d(\theta,\theta')^\alpha.
        \label{eq:proc}
    \end{equation}
    This can be summarized by saying that $\{\ell(x,y;\theta): \theta \in \Theta\}$ is a stochastic process with ($\alpha$-H\"older) $q-$Orlicz increments.
\end{assumption}
Notice that a combination of Assumptions \ref{ass:sub} and \ref{ass:stab} is in general not enough to guarantee \eqref{eq:proc}, and moreover Assumption \ref{ass:proc} requires the global H\"older stability of $\ell$ on $\Theta$, in contrast with the local requirement in Assumption \ref{ass:stab}. The following result can be seen as an extension of \cite[Theorem 5.22]{wainwright2019high}, adapted here for the present discussion.
\begin{proposition}
    \label{prop:chaining}
    Let $\Theta$ satisfy Assumption \ref{ass:compact} and $\ell$ satisfy Assumption \ref{ass:proc} with $q\in\{1,2\}$. Then, for all $r > 0$, 
\begin{equation}
\label{eq:chaining}
    \E[L(\hat{\theta})-L(\theta^\star)] \leq C_1 \frac{K_\ell}{\sqrt{m}} \int_{r^\alpha/4}^{D} \big( \log \mathcal{N}(\Theta,c^{1/\alpha})\big)^{1/q} dc + C_2 K_\ell r^{\alpha},
\end{equation}
being $C_1,C_2>0$ absolute constants, and $D\geq 1$ the diameter of $\Theta$. 
\end{proposition}
\begin{proof}
Due to Assumption \ref{ass:proc}, also $Z_\theta$ is a random process with ($\alpha$-H\"older) $q-$Orlicz increments, and
\[
\| Z_\theta - Z_{\theta'}\|_{\psi_q} \leq \frac{K_\ell}{\sqrt{m}}d(\theta,\theta')^\alpha.
\]
The starting point of the proof is similar to the one of Proposition \ref{prop:chaining}, although we need to center the process by considering the differences of its elements. 
It is easy to see, since $Z_\theta$ is zero-mean, that
\[
\sup_{\theta \in \Theta} |Z_\theta| \leq 2 \sup_{\theta,\theta' \in \Theta} |Z_\theta - Z_{\theta'}|.
\]
Moreover, if $\mathbb{U}$ is a $r-$covering set of $\Theta$, it holds
(cf. the proof of \cite[Proposition 5.17]{wainwright2019high})
\[
\sup_{\theta,\theta' \in \Theta}|Z_\theta - Z_{\theta'}| \leq 2 \sup_{\theta,\theta':d(\theta,\theta')\leq r}|Z_\theta - Z_{\theta'}| + 2 \max_{\theta,\theta' \in \mathbb{U}} |Z_{\theta}-Z_{\theta'}|.
\]
Hence, recalling \eqref{eq:representativeness}, it holds
\[
\E[L(\hat{\theta})-L(\theta^\star)] \leq 8 \E\bigg[ \max_{\theta,\theta' \in \mathbb{U}} |Z_{\theta}-Z_{\theta'}| \bigg] + 8 \E\bigg[ \sup_{\theta,\theta':d(\theta,\theta')\leq r}|Z_\theta - Z_{\theta'}| \bigg]
\]
The second term on the right-hand side can be treated in the same way as in Proposition \ref{prop:covering}, and (also thanks to \eqref{eq:sub-G-mom} and \eqref{eq:sub-exp-mom}) it leads to the second term on the right-hand side of \eqref{eq:chaining}. For the first term, instead, we follow closely the proof of \cite[Proposition 5.22]{wainwright2019high}, which already provides a complete proof for the case $q=2$, $\alpha=1$.
\\
Let $r_k=\left(\frac{D}{2^k}\right)^{1/\alpha}$ and consider the sets $\mathbb{U}_k$, minimal $r_k-$covering of $\mathbb{U}$. Notice that their cardinality $N_k$ is bounded by $\mathcal{N}(\Theta;r_k)$; moreover, since $\mathbb{U}$ is finite, $\mathbb{U}_k=\mathbb{U}$ whenever $r_k \leq r$: let $K$ be such that $r_K < r \leq r_{K-1}$ (from which it follows, in particular, that $D/2^{K-1} \geq r^\alpha$). Consider the map projecting each parameter on the closest element of the covering $\mathbb{U}_k$:
\[
\pi_k\colon \mathbb{U}\rightarrow \mathbb{U}_k, \quad \pi_k(\beta) = \argmin_{\beta' \in \mathbb{U}_k} d(\beta,\beta');
\]
by definition of $\pi_k$ and $\mathbb{U}_k$, for every $\beta \in \Theta$ it holds that $d(\beta, \pi_k(\beta))\leq r_k$. 
Now, proceeding exactly as in \cite[Proposition 5.22]{wainwright2019high}, it is easy to see that
\begin{equation}
    \label{eq:aux0}
\max_{\theta,\theta' \in \mathbb{U}} |Z_\theta - Z_{\theta'}| \leq \max_{\theta,\theta' \in \mathbb{U}_1} |Z_\theta - Z_{\theta'}| + 2 \sum_{k=2}^K \max_{\beta \in \mathbb{U}_k} |Z_\beta - Z_{\pi_{k-1}(\beta)}|
\end{equation}
Let us now take the expected values: the first term on the right-hand side can be bounded using the formula for the maxima of a finite collection of $q-$Orlicz random variables, \eqref{eq:sup_Orlicz}. Indeed, by Assumption \ref{ass:proc}, $Z_\theta - Z_{\theta'}$ is $q-$Orlicz with norm bounded by $\frac{K_\ell}{\sqrt{m}} d(\theta,\theta')^\alpha$, hence by $\frac{K_\ell}{\sqrt{m}} D$ (where we have also used that $D\geq 1$, so that $D^\alpha \leq D$). The cardinality of $\mathbb{U}_1$ satisfies $\#\mathbb{U}_1 \leq \mathcal{N}\left(\Theta,(D/2)^{1/\alpha}\right)$: hence, by \eqref{eq:sup_Orlicz},
\begin{equation}
\label{eq:aux1}
\E\bigg[ \max_{\theta,\theta' \in \mathbb{U}_1} |Z_\theta - Z_{\theta'}| \bigg] \leq \frac{C K_\ell}{\sqrt{m}} D \left( \log \mathcal{N}\big(\Theta,(D/2)^{1/\alpha}\big)\right)^{1/q}
\end{equation}
Instead, each addend in the last term of \eqref{eq:aux0} can be bounded, again via \eqref{eq:sup_Orlicz}, using the fact that $d(\beta,\pi_{k-1}(\beta))\leq (D2^{-k+1})^{1/\alpha}$:
\begin{equation}
\label{eq:aux2}
\E\bigg[ \max_{\beta \in \mathbb{U}_k} |Z_\beta - Z_{\pi_{k-1}(\beta)}| \bigg] \leq \frac{C K_\ell}{\sqrt{m}} \frac{D}{2^{k-1}}  \left( \log \mathcal{N}\big(\Theta,(D/2^k)^{1/\alpha}\big)\right)^{1/q}.
\end{equation}
Collecting the results in \eqref{eq:aux0}, \eqref{eq:aux1} and \eqref{eq:aux2}, we get
\begin{equation}
    \label{eq:aux4}
\E\bigg[ \max_{\theta,\theta' \in \mathbb{U}} |Z_\theta - Z_{\theta'}| \bigg] \leq 
2\frac{C K_\ell}{\sqrt{m}} D \sum_{k=1}^K 2^{-k} \left( \log \mathcal{N}\big(\Theta,(D/2^k)^{1/\alpha}\big)\right)^{1/q}
\end{equation}
Consider now the function $n(b) = \left( \log \mathcal{N}(\Theta,b^{1/\alpha})\right)^{1/q}$: since $\mathcal{N}(\Theta,b)$ is non-increasing, also $n$ is a (continuous) non-increasing function of $b$. For this reason,
\[
\begin{gathered}
n(b) \leq \frac{2}{b} \int_{b/2}^b n(c)dc, \quad \text{i.e.,} \\
\left( \log \mathcal{N}\big(\Theta,(D/2^k)^{1/\alpha}\big)\right)^{1/q}
\leq 
\frac{2^{k+1}}{D}\int_{D/2^{k+1}}^{D/2^k} \left( \log \mathcal{N}(\Theta,c^{1/\alpha})\right)^{1/q} dc
\end{gathered}
\]
and as a consequence
\begin{equation}
    \label{eq:aux5}
\E \bigg[\max_{\theta,\theta' \in \mathbb{U}} |Z_\theta - Z_{\theta'}| \bigg] \leq 
4\frac{C K_\ell}{\sqrt{m}} \sum_{k=1}^K \int_{D/2^{k+1}}^{D/2^k} \left( \log \mathcal{N}(\Theta,c^{1/\alpha})\right)^{1/q} dc.
\end{equation}
Summing up the integrals in \eqref{eq:aux5} allows to conclude the proof, remembering that $D/2^{K+1} \geq r^{\alpha}/4$, hence the overall integration interval is contained within $(r^{\alpha}/4,D)$.
\end{proof}
The result in \eqref{eq:chaining} holds for every $r>0$ and, by taking the limit, also for $r=0$, provided that the function $\big( \log \mathcal{N}(\Theta,c^{1/\alpha})\big)^{1/q}$ is integrable on $(0,D)$ - in an improper sense, as in $c=0$ it shows a singularity. 
Notice that \cite[Theorem 5.36]{wainwright2019high} already describes a chaining procedure for $q-$Orlicz processes, also deriving a bound in probability, but only for the case $r=0$ and $\alpha =1$.

\section{Examples}
\label{sec:examples}

The sample error estimates reported in Proposition \ref{prop:covering} and \ref{prop:chaining} rely on several assumptions: the goal of this section is to provide a general framework and concrete examples under which those conditions are satisfied. \\
Section \ref{ssec:ex-Theta} focuses on Assumption \ref{ass:compact}: namely, the compactness of the set of parameters $\Theta$, providing explicit versions of the convergence rates \eqref{eq:covering} and \eqref{eq:chaining}. Section \ref{ssec:ex-loss} is instead dedicated to the verification of Assumptions \ref{ass:sub} and \ref{ass:stab}, or of Assumption \ref{ass:proc}, both concerning the random process $\{ \ell(x,y;\theta), \ \theta \in \Theta\}$. Finally, two examples of learned reconstruction strategies satisfying all the proposed assumptions are shown in \ref{ssec:example}, as an informative application.

\subsection{Compact classes of parameters}
\label{ssec:ex-Theta}

The simplest case of a compact set $\Theta$ is represented by a bounded and closed subset of a finite-dimensional space. Without loss of generality, let us consider a closed ball:
\begin{equation}
\Theta = \{ \theta \in \R^{d}: \| \theta \|\leq D \}
    \label{eq:compact_finite}
\end{equation}
Obviously, $\Theta$ is compact with respect to the Euclidean distance $d(\theta,\theta')=\| \theta - \theta'\|$; moreover, (see, e.g., \cite[Example 27.1]{shalev2014}),
\begin{equation}
\mathcal{N}(\Theta,r) \leq \left(\frac{2D\sqrt{d}}{r}\right)^d.
    \label{eq:covering_finite}
\end{equation}
Let us now substitute the expression of $\log\mathcal{N}$ in Propositions \ref{prop:covering} and \ref{prop:chaining} to get more specific convergence rates. To ease the notation, let us not keep track of any constant which is independent of $m$ and $d$.
\begin{cor}
\label{cor:finite_dim}
Let $\Theta$ be as in \eqref{eq:compact_finite}. Then, $\exists m_0 >0$ such that, for $m \geq m_0$:
\begin{align}
    &\text{under Assumptions \ref{ass:sub} and \ref{ass:stab},} \quad \E[L(\hat{\theta})-L(\theta^\star)] \lesssim \frac{\log(m)^{1/q}}{\sqrt{m}}d^{1/q}; \label{eq:finite_cover} \\
    &\text{under Assumption \ref{ass:proc},} \qquad \qquad \E[L(\hat{\theta})-L(\theta^\star)] \lesssim \frac{1}{\sqrt{m}}\left(d\log(d)\right)^{1/q} \label{eq:finite_chain}.
\end{align}
\end{cor}
\begin{proof}
When substituting expression \eqref{eq:covering_finite} in \eqref{eq:covering}, one gets
\[
\E[L(\hat{\theta})-L(\theta^\star)] \lesssim \frac{1}{\sqrt{m}}\left(d \log\bigg(\frac{\sqrt{d}}{r}\bigg)\right)^{1/q} + r^\alpha
\]
The optimal choice of $r$ would be determined by imposing that the two terms on the right-hand side are asymptotically equivalent as $m \rightarrow \infty$: though, an analytic expression of $r(m)$ would imply the use of the Lambert $W$ function, and result in a non-informative bound. For this reason, let us pick a (slightly) suboptimal choice: $\log\left(\frac{1}{r}\right) \lesssim \log(m)$, for which the first term imposes the rate reported in \eqref{eq:finite_cover}. Instead, substituting \eqref{eq:covering_finite} in \eqref{eq:chaining}, one realizes that the function
\[
\big(\log\mathcal{N}(\Theta,c^{1/\alpha})\big)^{1/q} \lesssim \big(d \log(d)\big)^{1/q} + d^{1/q} \log\bigg(\frac{1}{c}\bigg)^{1/q},
\]
is (improperly) integrable on $(0,D)$. Thus, setting $r=0$ (and disregarding the term depending on $D$), we obtain the bound \eqref{eq:finite_chain}.  
\end{proof}
Observing the rates in Corollary \ref{cor:finite_dim}, one easily sees that the ones based on chaining yield a better asymptotical rate, which also matches the results on finite-dimensional cases obtained, e.g., via VC-dimension.
\par
Another example of prominent importance is the following one: let $\mathcal{X}$ and $\mathcal{X}_*$ be Banach spaces (with norms $\|\cdot\|_{\mathcal{X}}$ and $\|\cdot\|_{*}$, respectively) such that $\mathcal{X}_* \subset \mathcal{X}$ and the embedding $\iota: \mathcal{X}_* \rightarrow \mathcal{X}$ is compact. Moreover, assume that the singular values $s_k(\iota)$ of $\iota$ have the following polynomial decay:
\begin{equation}
    s_k(\iota) \lesssim k^{-s},\qquad s>0.
    \label{eq:iota_decay}
\end{equation}
The value of $s$ allows to \textit{quantify} the compactness of the embedding of $\mathcal{X}_*$ in $\mathcal{X}$. It is often possible to assess conditions like \eqref{eq:iota_decay} for functional spaces: for example, if $\mathcal{X}=H^{\sigma_0}(\mathbb{T}^d)$ and $\mathcal{X}_*= H^{\sigma_1}(\mathbb{T}^d)$ are Sobolev spaces on the $d-$dimensional torus with $s_1 > s_0$, \eqref{eq:iota_decay} is satisfied with $s = \frac{\sigma_1-\sigma_0}{d}$. 
\\
Under this assumption, one can define (for simplicity, let us set $D=1$) 
\begin{equation}
\Theta = \{ \theta \in \mathcal{X}_*: \| \theta \|_{*} \leq 1 \};
    \label{eq:compact_inf}
\end{equation}
it is possible to show, via the tool of \textit{entropy numbers} (see \cite{cast90}, or \cite[Lemma A.8]{alberti2021learning}), that, in the metric $d(\theta,\theta') = \|\theta-\theta'\|_{\mathcal{X}}$,
\begin{equation}
\log \mathcal{N}(\Theta,r) \lesssim r^{-1/s}.
    \label{eq:covering_inf}
\end{equation}
We can now provide explicit expressions of the rates in Propositions \ref{prop:covering} and \ref{prop:chaining} (for the sake of notation, let us not keep track of any constant which is independent of $m$ and $s$).
\begin{cor}
\label{cor:inf_dim}
Let $\Theta$ be as in \eqref{eq:compact_finite}. Then, $\exists m_0 >0$ such that, for $m \geq m_0$:
\begin{align}
    &\text{under Assumptions \ref{ass:sub} and \ref{ass:stab},} \quad \quad\ \ \E[L(\hat{\theta})-L(\theta^\star)] \lesssim m^{-\frac{1}{2}\left(1-\frac{1}{1+\alpha s q}\right)}; \label{eq:inf_cover} \\
    &\text{under Assumption \ref{ass:proc}, if $s\leq \frac{1}{\alpha q}$,} \quad \E[L(\hat{\theta})-L(\theta^\star)] \lesssim m^{-\frac{1}{2}
    \alpha^2 s q
    } \label{eq:inf_chain1};\\
    &\text{under Assumption \ref{ass:proc}, if $s > \frac{1}{\alpha q}$,} \quad \E[L(\hat{\theta})-L(\theta^\star)] \lesssim m^{-\frac{1}{2}} \label{eq:inf_chain2}.
\end{align}
\end{cor}
\begin{proof}
    Substituting \eqref{eq:covering_inf} in \eqref{eq:covering}, one gets
    \[
    \E[L(\hat{\theta})-L(\theta^\star)] \lesssim \frac{1}{\sqrt{m}}r^{-1/sq} + r^\alpha,
    \]
    and the optimal rate is achieved via the choice $r \sim m^{-\frac{sq}{2(1+\alpha sq)}}$, which balances the two terms in the right-hand side and yields \eqref{eq:inf_cover}. Instead, inserting \eqref{eq:covering_inf} in \eqref{eq:chaining}, the resulting function
    \[
    \big(\log\mathcal{N}(\Theta,c^{1/\alpha})\big)^{1/q} \lesssim c^{-\frac{1}{\alpha s q}}
    \]
    can be integrated up to $0$ only if $\alpha s q > 1$. If that is the case, the optimal rate $m^{-1/2}$ is achieved for the choice $r=0$. Otherwise, for any $r>0$, 
    \[
    \E[L(\hat{\theta})-L(\theta^\star)] \lesssim \frac{1}{\sqrt{m}}r^{\alpha\left(1-\frac{1}{\alpha sq}\right)} + r^\alpha,
    \]
    and the optimal rate is achieved by the choice $r \sim m^{-\frac{1}{2}\alpha sq }$, leading to \eqref{eq:inf_chain1}.
\end{proof}
Some additional considerations may help to interpret the bounds proved in Corollary \ref{cor:inf_dim}: if the embedding $\iota$ is sufficiently compact, the resulting convergence rate saturates to the optimal value $m^{-1/2}$. If $s \in \left(0,\frac{1}{\alpha q} \right)$, one may take advantage either of the covering-based bound \eqref{eq:inf_cover} or of the chaining-based bound \eqref{eq:inf_chain1}: by comparing the terms $1 - \frac{1}{1+\alpha sq}$ and $\alpha^2 sq$ (which are both smaller than $1$ by the assumption on $s$), one realizes that the covering-base rate is faster only for smaller values of $s$, namely for $s \in \left(0, \frac{1-\alpha}{\alpha^2 q}\right)$. Notice finally that, in the case of Lipschitz stable loss ($\alpha=1$), the chaining technique performs better for any value of $s$.

\subsection{Stable and $q-$Orlicz losses}
\label{ssec:ex-loss}

This section describes a more specific scenario in which Assumptions \ref{ass:sub} and \ref{ass:stab}, or Assumption \ref{ass:proc}, are satisfied. This accounts for further specifying the properties of the loss $\ell(x,y;\theta)$, hence of the random variables $x,y$, of the parametric operator $R_\theta$ and of the metric $\tilde{\ell}$ in \eqref{eq:loss}. Let us only restrict ourselves to the case of quadratic loss:
\[
\ell(x,y;\theta) = \frac{1}{2} \| R_\theta(y) - x \|_X^2, 
\]
from which the following bounds immediately follow:
\begin{align}
    \ell(x,y;\theta) &\leq \| R_\theta(y)\|^2_X + \| x\|_X^2 \label{eq:bound_ell}  \\
    |\ell(x,y;\theta) \! \shortminus \! \ell(x,y;\theta')| &\leq \frac{1}{2} \| R_\theta(y) \! \shortminus \! R_{\theta'}(y)\|_X \big(\| R_\theta(y) \! \shortminus \! x \|_X + \|R_{\theta'}(y)\! \shortminus \! x\|_X \big) \label{eq:bound_ell_diff} 
\end{align}
Then, the verification of Assumptions \ref{ass:stab} and \ref{ass:proc} translates into a requirement on the operators $R_\theta$. In particular, suppose that $R_\theta$ is locally H\"older stable in $\theta$, with a constant depending on $y$ only through $\|y\|_Y$: in particular, for some $r_0,L_R,L'_R>0$,
\begin{equation}
 \| R_{\theta}(y) - R_{\theta'}(y)\|_X \leq (L_R \|y\|_Y + L'_R)\  d(\theta,\theta')^\alpha \qquad
    \forall y, \ \forall \theta,\theta': d(\theta,\theta')\leq r_0 
    \label{eq:stab_reg}
\end{equation}
Notice that, if $\Theta$ is assumed to be compact as in Assumption \ref{ass:compact}, then a local stability estimate as \eqref{eq:stab_reg} also translates into a global stability one on $\Theta$, although with larger constants $L_R,L'_R$. 
Assume also that $R_\theta$ is sub-linear in $y$, uniformly in $\theta$:
\begin{equation}
 \| R_{\theta}(y) \|_X \leq M_R \|y\|_Y + M'_R \qquad
    \forall y, \ \forall \theta 
    \label{eq:bound_reg}
\end{equation}
Let us point out that, if \eqref{eq:stab_reg} and Assumption \ref{ass:compact} hold true, then \eqref{eq:bound_reg} is verified for all $\theta$ if it holds just for a single $\theta_0 \in \Theta$.
\\
Finally, let us assume that the random variables $x$ and $y$ have sub-Gaussian or sub-exponential square norms, namely:
\begin{equation}
\big\| \| x \|_X^2 \big\|_{\psi_q} \leq K_x, \quad
\big\| \| y \|_Y^2 \big\|_{\psi_q} \leq K_y, \quad q \in \{1,2\}
    \label{eq:Orlicz_norm_xy}
\end{equation}
Since $x,y$ are random variables taking values on (possibly infinite-dimensional) Hilbert spaces $X,Y$, satisfying \eqref{eq:Orlicz_norm_xy} might not be straightforward. Nevertheless, it is verified in the following two important scenarios:
\begin{itemize}
    \item if $x$ and $y$ are bounded random variables, i.e., the support of their joint distribution $\rho$ is a bounded subset of $X \times Y$. In this case, both $\|x\|_X^2$ and $\|y\|_Y^2$ are bounded, real-valued random variables, thus sub-Gaussian ($q=2$). As already pointed out, the boundedness assumption could be natural for the ground truth variable $x$, due to state constraints (e.g. non-negativity, upper bounds), but might be restrictive for $\varepsilon$.
    \item if $x$ and $y$ are sub-Gaussian random variables on $X$, then \eqref{eq:Orlicz_norm_xy} is satisfied with $q=1$. More precisely, let us recall that a square-integrable random variable $w$ on a Hilbert space $W$ is said to be sub-Gaussian with norm $K_w$ if it satisfies
    \[
    \| \langle w,v \rangle_W \|_{\psi_2} \leq K_w \quad \forall v \in W;
    \]
    for simplicity, and without loss of generality, let us assume that $\E[w]=0$. Then, as shown e.g. in the proof of \cite[Lemma A.10]{alberti2021learning}, $\| w\|_W$ is a real-valued sub-Gaussian random variable, with norm bounded by 
    \[
    \big \| \| w\|_W \big\|_{\psi_2} \leq K_w \bigg( \sqrt{\tr(\Sigma_w)}  + \sqrt{ \| \Sigma_w\|}\bigg),
    \]
    where $\| \Sigma_w\|$ is the operator norm of the covariance of $w$ and the trace $\tr(\Sigma_w)$ is finite since $w$ is square-integrable. In conclusion, thanks to \cite[Lemma 2.7.6]{vershynin2018high}, $\| w\|_W^2$ is sub-exponential and 
    \[
    \big \| \| w\|_W^2 \big\|_{\psi_1} \leq
    \big \| \| w\|_W \big\|_{\psi_2}^2
    \]
    \end{itemize}

Notice moreover that, according to the forward model $y=F(x)+\varepsilon$ and to Assumption \eqref{ass:stat}, the sub-Gaussianity of $y$ can be deduced, e.g., by the sub-Gaussianity of $x$ and $\varepsilon$, provided that $F$ is sublinear.
\par
The following result draws the obvious outcome of the proposed discussion.
\begin{proposition}
    Let the random variables $x,y$ satisfy \eqref{eq:Orlicz_norm_xy}; let $R_\theta$ be a parametric family of operators verifying \eqref{eq:stab_reg} and \eqref{eq:bound_reg}. Then, Assumptions \ref{ass:sub} and \ref{ass:stab} are verified. Moreover, if \eqref{eq:stab_reg} holds globally in $\Theta$, also Assumption \ref{ass:proc} is verified.
\end{proposition}

In conclusion, the verification of \eqref{eq:stab_reg} and \eqref{eq:bound_reg} is the task of anyone proposing a new learned reconstruction strategy. In short, the H\"older stability of the regularized solution with respect to the parameters (together with the compactness of the parameter class, and with the statistical assumptions on $x$ and $\varepsilon$) is enough to prove the generalization properties of the proposed strategy and to obtain sample error estimates.

\subsection{Two concrete examples}
\label{ssec:example}

The stability and boundedness conditions on $R_\theta$ expressed in \eqref{eq:stab_reg} and \eqref{eq:bound_reg} are met by many relevant families of learned reconstruction strategies. Let us consider one example in learned variational regularization and one in proximal-based techniques.

\paragraph*{Learned variational regularization: Elastic-Net\\} 
For a linear inverse problem $y=Ax+\varepsilon$, consider the following regularizer, inspired by the well-known Elastic-Net functional (\cite{zou2005regularization,de2009elastic}): for $\eta >0$,
\begin{equation}
R_\theta(y) = \argmin_{x \in X}\bigg\{ \frac{1}{2} \| A x - y \|_Y^2 + g_\theta(x) + \eta \| x\|_X^2 \bigg\},
    \label{eq:elasticnet}
\end{equation}
being $g_\theta$ a family of functionals from $X$ to $\R$, parametrized in $\theta \in \Theta$.

\begin{proposition} \label{prop:elastic_net}
Suppose that $g_\theta$ satisfies:
\begin{enumerate}[i)]
    \item for all $\theta \in \Theta$, $g_\theta\colon X \rightarrow\R$ is convex, coercive and lower-semicontinuous;
    \item uniform lower bound: $g_\theta(y) \geq 0$ for all $y$ and $\theta$;
    \item $g_\theta(0) \leq M_g$ for all $\theta$;
    \item uniform H\"older stability: $|g_\theta(y) - g_{\theta'}(y)| \leq C_g \|y\|^2_Y d(\theta,\theta')^{2\alpha}$.
\end{enumerate}
Then, $R_\theta$ in \eqref{eq:elasticnet} is well-defined and satisfies \eqref{eq:stab_reg} and \eqref{eq:bound_reg}. 
\end{proposition}
\begin{proof}
The existence and uniqueness of $R_{\theta}(y)$ for each $y,\theta$ is guaranteed by the fact that the functional to be minimized in \eqref{eq:elasticnet} is not only lower-semicontinuous, coercive on $X$ and bounded from below, but also strongly convex. 
\\
Let $p_\theta = R_\theta(y)$: then, the uniform bound \eqref{eq:bound_reg} is obtained by evaluating the functional \eqref{eq:elasticnet} in $p_\theta$ and in $0$:
\[
\frac{1}{2} \| A p_\theta - y \|_Y^2 + g_\theta(p_\theta) + \eta \| p_\theta \|_X^2 \leq \frac{1}{2} \| y \|_Y^2 + g_\theta(0);
\]
by hypotheses $ii)$ and $iii)$ it holds $\|p_\theta\|_X^2 \leq \frac{1}{2\eta}\|y\|_Y^2 + M_g$.
Moreover, due to the convexity of $g_\theta$, $p_\theta$ verifies the following optimality conditions:
\[
-\langle A^*(Ap_\theta-y),p\rangle_X - \eta p_\theta \in \partial g_{\theta}(p_\theta) \quad \forall p \in X.
\]
For a different $\theta'$, let $p_{\theta'} = R_{\theta'}(y)$: evaluating the optimality conditions for $p_\theta$ in $p = p_\theta$ and optimality conditions for $p_{\theta'}$ in $p=p_\theta$ and subtracting them, one gets
\[
\eta \| p_\theta - p_{\theta'}\|_X^2 + \| A(p_\theta- p_{\theta'})\|_Y^2 \leq g_{\theta}(p_{\theta'}) - g_{\theta}(p_{\theta}) + g_{\theta'}(p_{\theta}) - g_{\theta'}(p_{\theta'}).
\]
Finally, by hypothesis $iv)$, we get
\[
\| p_\theta - p_{\theta'}\|_X^2 \leq \frac{C_g}{\eta} (\|p_\theta\|^2_X + \|p_\theta\|^2_X)\ d(\theta,\theta')^{2\alpha}.
\]
\end{proof}
Examples of $g_\theta$ satisfying the hypotheses of Proposition \ref{prop:elastic_net} can be obtained by means, e.g., of Input Convex Neural Networks (see \cite{amos2017input}). Alternatively, one might consider the following parametric functionals:
\[
g_{(h,B)}(y) = \| By-h\|_{X}^{2\alpha},
\]
for $\alpha \in (0,1]$, where $\theta = (h,B)$, $h \in X$, $B\colon X\rightarrow X$. Notice that the case $\alpha=1$, without the need of the Elastic-net regularization ($\eta=0$) has already been treated in \cite{alberti2021learning}: the sample error estimate in \cite[Theorem 4.1]{alberti2021learning}, for bounded random variables, can be compared with Corollary \ref{cor:inf_dim}, setting
$\alpha=1$ and $q=2$.

\paragraph*{Learned fixed-point method: stable contractive maps\\}
Again for a linear inverse problem $y=Ax+\varepsilon$, consider a family of operators defined via fixed points of the maps $\phi_\theta$, suitably parametrized:
\begin{equation}
    \label{eq:fixed_point}
    R_\theta(y) = p_\theta \quad \text{s.t.} \quad p_\theta = \phi_\theta(p_\theta;y).
\end{equation}
These maps occur in Deep Equilibrium techniques, and also in plug-and-play schemes, although in this latter case they are not directly trained to minimize the expected loss $L(\theta)$. The formula \eqref{eq:fixed_point} can also be linked to variational regularization, whenever it is possible to interpret $\phi_\theta$ as the proximal map of a (convex) functional, see \cite{gribonval2020characterization}.
Lipschitz stability of fixed points of contractive maps is established, e.g., in \cite{adly2014one}, as a corollary of Lim's lemma (\cite{lim1985fixed}). For $y \in Y$, denote by $\Phi(\theta,z)=\phi_\theta(z;y)$.
\begin{proposition}\cite[Theorem 2]{adly2014one} Suppose that:
\begin{enumerate}[i)]
    \item $\Phi(\theta,\cdot)$ is Lipschitz continuous with constant $0 < L_z < 1$, uniformly in $\theta$;
    \item $\Phi(\cdot,z)$ is Lipschitz continuous with constant $L_\theta$, uniformly in $z$.
\end{enumerate}
Then, the fixed point $p_\theta$: $p_\theta = \Phi(\theta,p_\theta)$ is a Lipschitz continuous function of $\theta$, with a constant $\frac{L_\theta}{1-L_z}$.
\label{prop:fixed_point}    
\end{proposition}
If the above assumptions hold uniformly in $y$, then also the Lipschitz continuity of $R_\theta(y)$ with respect to $\theta$ is uniform in $y$, hence \eqref{eq:stab_reg} holds with $\alpha=1$. Requirement $i)$, namely, the contractivity of the learned maps $\phi_\theta(\cdot;y)$ is a common feature in Deep Equilibrium schemes (verified for different algorithms, e.g., in \cite[Section 5]{gilton2021deep}). 
\\
The Lipschitz stability in $\theta$ depends, instead, on the employed architecture in $\phi_\theta$, and it is usually verified in feed-forward neural networks with Lipschitz activation function, since the composition of Lipschitz maps is again Lipschitz.

\section{Conclusions and outlook}
The core discussion of this work revolves around sample error estimates for learned reconstruction methods in inverse problems. To provide a comprehensive and fruitful approach, a general statistical learning setup has been introduced, according to which a broad class of techniques can be interpreted. Then, two different strategies have been employed to provide the desired bounds, and their outcome has been discussed and further specified by means of relevant examples.
\par
In the current formulation, the main results proposed in this work can be employed by the inverse problems community to derive sample error estimates for many existing techniques, as well as for new ones. A desired outcome of this dissertation is, therefore, to foster the interest of the community in the direction of generalization guarantees for data-driven methods.
\par
Some natural, and advisable, extensions of the present work can be sketched on several levels. First, the proposed technique for sample error estimates can be compared to other approaches in empirical risk minimization and extended to larger classes of random variables. Second, a similar analysis on the \textit{approximation} error would be mandatory to have a full picture of the generalization error estimates for a learned reconstruction technique. Finally, modifications of the original expected-loss minimization problem should be envisaged, such as the inclusion of regularization terms (leading to Regularized Risk Minimization) or the extension in the direction of adversarial learning.

\section*{Acknowledgements}
The author is grateful to Ernesto De Vito for introducing him to the study of statistical learning, and also to Giovanni S. Alberti, Matti Lassas, and Matteo Santacesaria for their valuable collaboration which has inspired the present dissertation.
\par
The research of the author has been funded by PNRR - M4C2 - Investimento 1.3. Partenariato Esteso PE00000013 - ``FAIR - Future Artificial Intelligence Research'' - Spoke 8 ``Pervasive AI'', which is funded by the European Commission under the NextGeneration EU programme. The author acknowledges the support of ``Gruppo Nazionale per l’Analisi Matematica, la Probabilità e le loro Applicazioni'' of the ``Istituto Nazionale di Alta Matematica'' through project GNAMPA-INdAM, code CUP\_E53C22001930001.

\bibliographystyle{abbrv}
\bibliography{references}
\end{document}